\documentclass{article}

\usepackage[left=1in, right=1in, top=1in, bottom=1in]{geometry}

\usepackage[square,sort,comma,numbers]{natbib}

\usepackage{microtype}
\usepackage{graphicx}
\usepackage{subfigure}
\usepackage{array}
\usepackage{booktabs} 
\usepackage{amsfonts}
\usepackage{amssymb,amsmath,amsthm,bm}
\newtheorem{remark}{Remark}
\newtheorem{theorem}{Theorem}
\newtheorem{definition}{Definition}
\newtheorem{lemma}{Lemma}
\newtheorem{proposition}{Proposition}
\usepackage{soul}
\usepackage{lmodern}
\usepackage{booktabs}
\usepackage{threeparttable}
\usepackage{mathtools}
\usepackage{tablefootnote}

\usepackage{cuted}
\usepackage{lipsum, color}

\usepackage{verbatim}

\usepackage{hyperref}

\usepackage{algorithm}
\usepackage[noend]{algpseudocode}

\DeclareMathOperator{\DTFT}{\mathrm{DTFT}}

\def \Ib {{\mathbf{Ib}}}
\def \ub {{\mathbf{u}}}
\def \cM {{\mathcal{M}}}
\def \cS {{\mathcal{S}}}
\def \cR {{\mathcal{R}}}
\def \PP {{\mathbb{P}}}
\def \RR {{\mathbb{R}}}
\def \EE {{\mathbb{E}}}

\newcommand{\la}{\langle}
\newcommand{\ra}{\rangle}
\def \wb {{\mathbf w}} 
\def \nb {{\mathbf n}} 

\def \bx {{\bm x}}
\def \by {{\bm y}}
\def \Ib {{\mathbf{I}}}
\def \Ab {{\mathbf{A}}}
\def \Bb {{\mathbf{B}}}
\def \Ub {{\mathbf{U}}}
\def \bLambda {{\bm \Lambda}}

\DeclareMathOperator{\Tr}{Tr}

\usepackage{soul}

\title{DP-LSSGD: A Stochastic Optimization Method to Lift the Utility in Privacy-Preserving ERM}

\author{
Bao Wang \\
  Department of Mathematics\\
  University of California, Los Angeles\\
  \texttt{wangbaonj@gmail.com}\\
  \and
  Quanquan Gu \\
  Department of Computer Science\\
  University of California, Los Angeles\\
  \texttt{qgu@cs.ucla.edu}\\
  \and
  March Boedihardjo \\
  Department of Mathematics\\
  University of California, Los Angeles\\
  \texttt{march@math.ucla.edu} \\
  \and
  Farzin Barekat\\
  Department of Mathematics\\
  University of California, Los Angeles\\
  \texttt{fbarekat@math.ucla.edu} \\
  \and
  Stanley J. Osher \\
  Department of Mathematics\\
  University of California, Los Angeles\\
   \texttt{sjo@math.ucla.edu} \\
}

\begin{document}

\maketitle

\begin{abstract}
Machine learning (ML) models trained by differentially private stochastic gradient descent (DP-SGD) have much lower utility than the non-private ones. To mitigate this degradation, we propose a DP Laplacian smoothing SGD (DP-LSSGD) to train ML models with differential privacy (DP) guarantees. At the core of DP-LSSGD is the Laplacian smoothing, which smooths out the Gaussian noise used in the Gaussian mechanism. Under the same amount of noise used in the Gaussian mechanism, DP-LSSGD attains the same DP guarantee, but in practice, DP-LSSGD makes training both convex and nonconvex ML models more stable and enables the trained models to generalize better. The proposed algorithm is simple to implement and the extra computational complexity and memory overhead compared with DP-SGD are negligible. DP-LSSGD is applicable to train a large variety of ML models, including DNNs. The code is available at \url{https://github.com/BaoWangMath/DP-LSSGD}.
\end{abstract}

\section{Introduction}
Many released machine learning (ML) models are trained on sensitive data that are often crowdsourced or contain private information \citep{Yuen:2011,Feng:2017,Liu:2017}. With overparameterization, deep neural nets (DNNs) can memorize the private training data, and it is possible to recover them and break the privacy by attacking the released models \citep{Shokri:2017}. For example, Fredrikson et al. demonstrated that a model-inversion attack can recover training images from a facial recognition system \citep{Fredrikson:2015}. Protecting the private data is one of the most critical tasks in ML.

Differential privacy (DP) \citep{Dwork:2006-Basic} is a theoretically rigorous tool for designing algorithms on aggregated databases with a privacy guarantee. The idea is to add a certain amount of noise to randomize the output of a given algorithm such that the attackers cannot distinguish outputs of any two adjacent input datasets that differ in only one entry. 

For repeated applications of additive noise based mechanisms, many tools have been invented to analyze the DP guarantee for the model obtained at the final stage. These include the basic and strong composition theorems  
and their refinements \citep{Dwork:2006-Basic,Dwork:2010,Kairouz:2015}, the moments accountant \citep{Abadi:2016}, etc. Beyond the original notion of DP, there are also many other ways to define the privacy, e.g., local DP \citep{Duchi:2014}, concentrated/zero-concentrated DP \citep{Dwork:2016,Bun:2016}, and R\'enyi-DP (RDP) \citep{mironov2017renyi}.

Differentially private stochastic gradient descent (DP-SGD) reduces the utility of the trained models severely compared with SGD. As shown in Figure~\ref{Degrade-performance}, the training and validation losses of the logistic regression on the MNIST dataset increase rapidly when the DP guarantee becomes stronger. 
The convolutional neural net (CNN) \footnote{\url{github.com/tensorflow/privacy/blob/master/tutorials/mnist_dpsgd_tutorial.py}} trained by DP-SGD has much lower testing accuracy than the non-private one on the MNIST. We will discuss the detailed experimental settings in Section~\ref{Section-Experiments}. A natural question raised from such performance degradations is:

\emph{Can we improve DP-SGD, with negligible extra computational complexity and memory cost, such that it can be used to train general ML models with improved utility?
}


\begin{figure}[!ht]
\centering
\begin{tabular}{ccc}
\includegraphics[clip, trim=0.1cm 4.5cm 0.1cm 5cm, width=0.3\columnwidth]{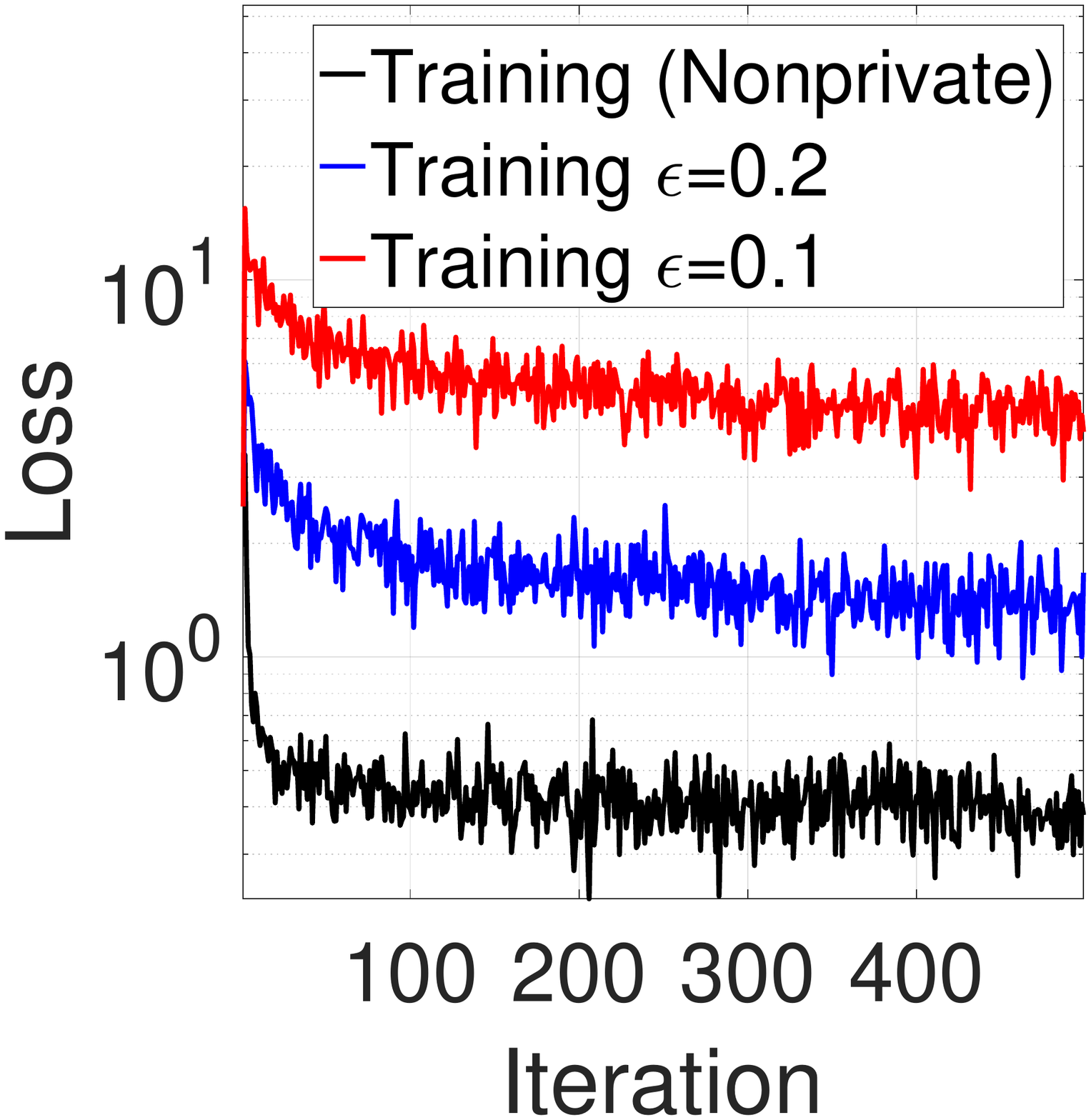}&
\includegraphics[clip, trim=0.1cm 4.5cm 0.1cm 5cm, width=0.3\columnwidth]{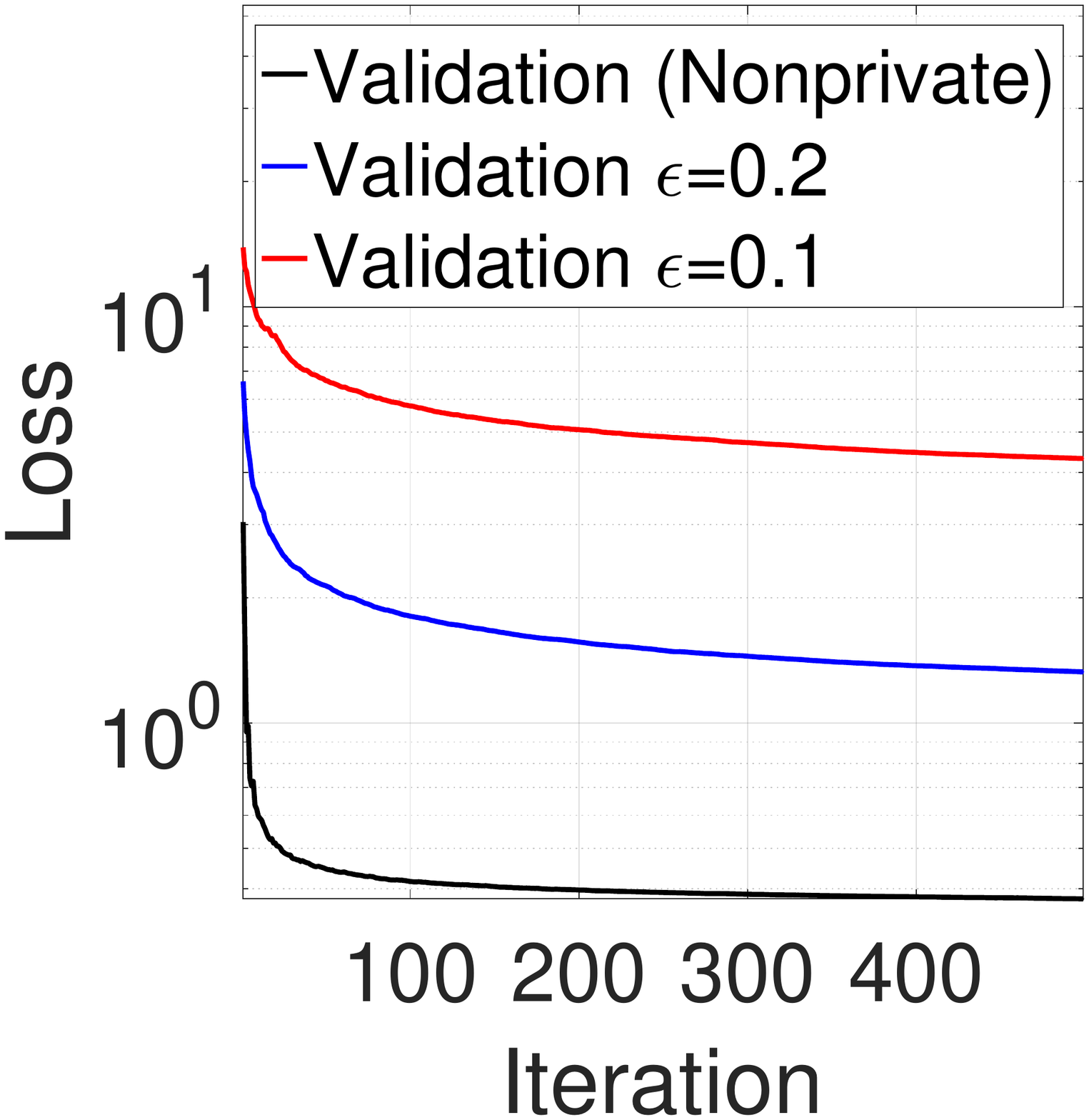}&
\includegraphics[clip, trim=0.1cm 4.5cm 0.1cm 5cm, width=0.3\columnwidth]{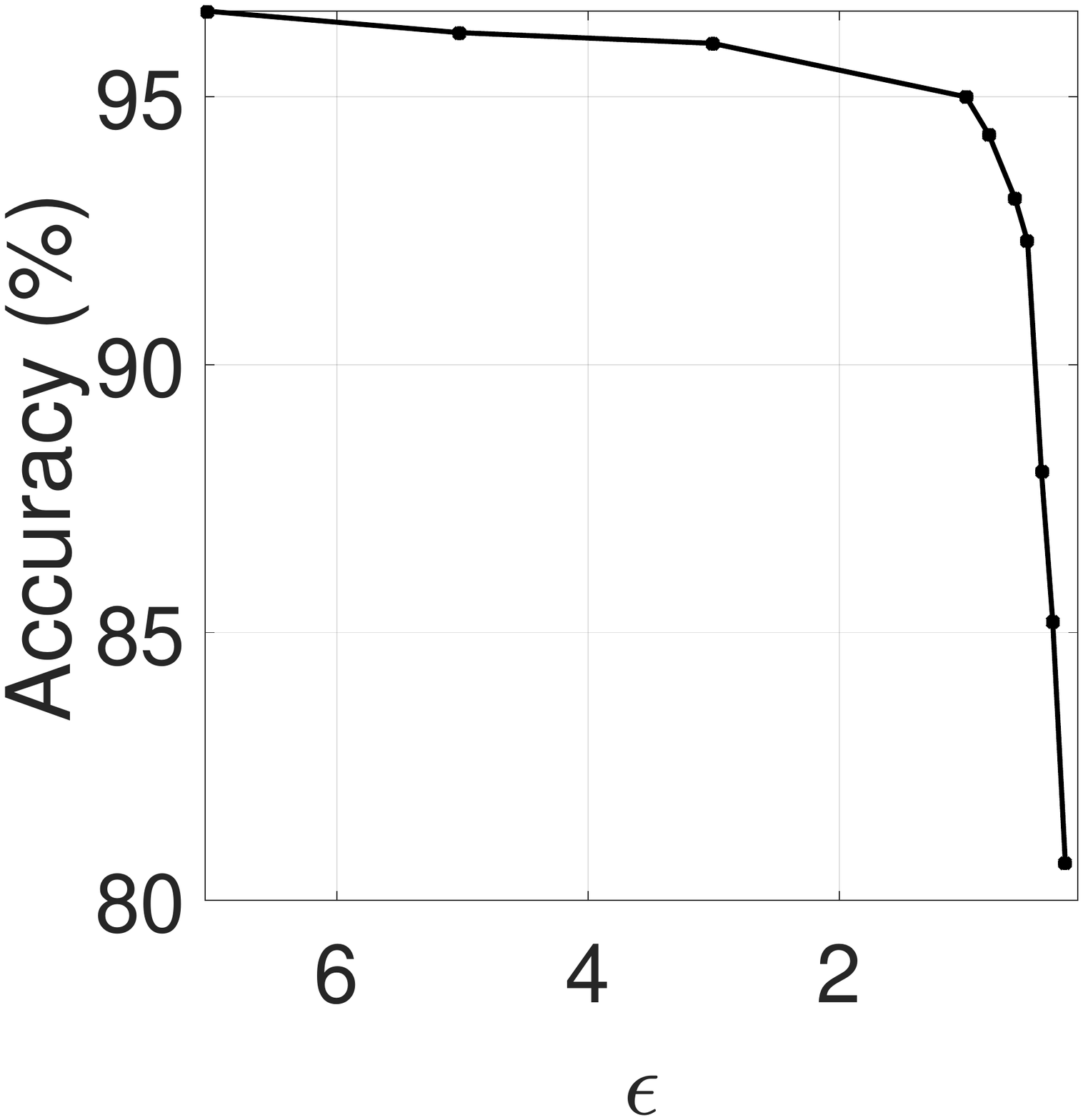}\\
\end{tabular}
\caption{Training (left) and validation (middle) losses of the logistic regression on the MNIST trained by DP-SGD with $(\epsilon, \delta=10^{-5})$-DP guarantee. (right): testing accuracy of a simple CNN on the MNIST trained by DP-SGD with $(\epsilon, \delta=10^{-5})$-DP guarantee.}
\label{Degrade-performance}
\end{figure}
We answer the above question affirmatively by proposing differentially private Laplacian smoothing SGD (DP-LSSGD) to improve the utility in privacy-preserving empirical risk minimization (ERM). 
DP-LSSGD leverages the Laplacian smoothing \citep{LSGD:2018} as a post-processing to smooth the injected Gaussian noise in the differentially private SGD (DP-SGD) to improve the convergence of DP-SGD in training ML models with DP guarantee.


\subsection{Our Contributions}
The main contributions of our work are 
highlighted as follows:
\begin{itemize}
\item We propose DP-LSSGD and prove its privacy and utility guarantees for convex/nonconvex optimizations. We prove that under the same privacy budget, DP-LSSGD achieves better utility, excluding a small term that is usually dominated by the other terms, than DP-SGD by a factor that is much less than one for convex optimization.

\item We perform a large number of experiments 
logistic regression and CNN to verify the utility improvement by using DP-LSSGD. Numerical results show that DP-LSSGD remarkably reduces training and validation losses and improves the generalization of the trained private models.
\end{itemize}

In Table~\ref{Utility:Comparison}, we compare the privacy and utility guarantees of DP-LSSGD and DP-SGD.
For the utility, the notation $\tilde{O}(\cdot)$ hides the same constant and log factors for each bound. The constants $d$ and $n$ denote the dimension of the model's parameters and the number of training points, respectively. The numbers $\gamma$ and $\beta$ are positive constants that are strictly less than one, and $D_0, D_\sigma, G$ are positive constants, which will be defined in Section~\ref{Theory-Section}.

\begin{table}[!ht]
\fontsize{8.0}{8.0}\selectfont
\centering
\begin{threeparttable}
\caption{Utility and Differential Privacy Guarantees. 
}\label{Utility:Comparison}
\begin{tabular}{cccccc}
\toprule[1.0pt]
Algorithm  &DP & Assumption  & Utility & Measurement & Reference \cr
\midrule[0.8pt]
DP-SGD    & $(\epsilon, \delta)$ & convex         & 
$\Tilde{\mathcal{O}}\left(\frac{\sqrt{(D_0+G^2)d}}{(\epsilon n)}\right)$ & optimality gap&
\cite{Bassily:2014} \cr
DP-SGD    & $(\epsilon, \delta)$ & nonconvex         & $\Tilde{\mathcal{O}}\left(\sqrt{d}/(\epsilon n)\right)$
 & $\ell_2$-norm of gradient &
\cite{Zhang:2017} \cr
DP-LSSGD  & $(\epsilon, \delta)$ & convex         
&$\Tilde{\mathcal{O}}\left(\frac{\sqrt{\gamma(D_\sigma+G^2) d}}{(\epsilon n)}\right)$ & optimality gap
& {\bf This Work} \cr
DP-LSSGD  & $(\epsilon, \delta)$ & nonconvex        &  
$\Tilde{\mathcal{O}}\left(\sqrt{\beta d}/(\epsilon n)\right)$
\tnote{1} &  $\ell_2$-norm of gradient & {\bf This Work} \cr
\bottomrule[1.0pt]
\end{tabular}
\begin{tablenotes}
\item[1] Measured in the norm induced by $\Ab_\sigma^{-1}$, we will discuss this in detail in Section~\ref{Section-Experiments}.
\end{tablenotes}
\end{threeparttable}
\end{table}

\subsection{Related Work}
There is a massive volume of research over the past decade on designing algorithms for privacy-preserving ML. Objective perturbation, output perturbation, and gradient perturbation are the three major approaches to perform ERM with a DP guarantee. 
\cite{Chaudhuri:2008,Chaudhuri:2011} considered both output and objective perturbations for privacy-preserving ERM, and gave theoretical guarantees for both privacy and utility for logistic regression and SVM. \cite{Song:2013} numerically studied the effects of learning rate and batch size in DP-ERM. \cite{Wang:2016} studied stability, learnability and other properties of DP-ERM. \cite{Lee:2018} proposed an adaptive per-iteration privacy budget in concentrated DP gradient descent. 
The utility bound of DP-SGD has also been analyzed for both convex and nonconvex smooth objectives \citep{Bassily:2014,Zhang:2017}. \cite{Jayaraman:2018} analyzed the excess empirical risk of DP-ERM in a distributed setting. Besides ERM, many other ML models have been made differentially private. These include: clustering \citep{Su:2015,Wang:2015-Clustering,Balcan:2017}, matrix completion \citep{Jain:2018}, online learning \citep{Jain:2012}, sparse learning \citep{talwar2015nearly,wang2019differential}, and topic modeling \citep{Park:2016}. \cite{Gilbert:2017} exploited the ill-conditionedness of inverse problems to design algorithms to release differentially private measurements of the physical system. 

\cite{Shokri:2015} proposed distributed selective SGD to train deep neural nets (DNNs) with a DP guarantee in a distributed system, however, the obtained privacy guarantee was very loose. \cite{Abadi:2016} considered applying DP-SGD to train DNNs in a centralized setting. They clipped the gradient $\ell_2$ norm to bound the sensitivity and invented the moment accountant to get better privacy loss estimation. \cite{Papernot:2017} proposed Private Aggregation of Teacher Ensembles/PATE based on the semi-supervised transfer learning to train DNNs, and this framework improves both privacy and utility on top of the work by \cite{Abadi:2016}.
Recently \cite{Papernot:2018} introduced new noisy aggregation mechanisms for teacher ensembles that enable a tighter theoretical DP guarantee. The modified PATE is scalable to the large dataset and applicable to more diversified ML tasks. 

Laplacian smoothing (LS) can be regarded as a denoising technique that performs post-processing on the Gaussian noise injected stochastic gradient. Denoising has been used in the DP 
earlier:
Post-processing can enforce consistency of contingency table releases \citep{barak2007privacy} and leads to accurate estimation of the degree distribution of private network \citep{hay2009accurate}.
\cite{nikolov2013geometry} showed that 
post-processing by projecting linear regression solutions, when the ground truth solution is sparse, to a given $\ell_1$-ball can remarkably reduce the estimation error. \cite{bernstein2017differentially} used Expectation-Maximization to denoise a class of graphical models' parameters. \cite{balle2018improving} showed that in the output perturbation based differentially private algorithm design, denoising dramatically improves the accuracy of the Gaussian mechanism in the high-dimensional regime.
To the best of our knowledge, we are the first to design a denoising technique on the Gaussian noise injected gradient to improve the utility of the trained private ML models.

\subsection{Notation}
We use boldface upper-case letters $\Ab$, $\Bb$ to denote matrices and boldface lower-case letters $\bx$, $\by$ to denote vectors. For vectors $\bx$ and $\by$ and positive definite matrix $\Ab$, we use $\|\bx\|_2$ and $\|\bx\|_\Ab$ to denote the $\ell_2$-norm and the induced norm by $\Ab$, respectively; $\la\bx, \by\ra$ denotes the inner product of $\bx$ and $\by$; and $\lambda_i(\Ab)$ denotes the $i$-th largest eigenvalue of $\Ab$. We denote the set of numbers from $1$ to $n$ by $[n]$. $\mathcal{N}(\mathbf{0}, \Ib_{d\times d})$ represents $d$-dimensional standard Gaussian.
\subsection{Organization}
This paper is organized in the following way: In Section~\ref{Algorithm-Section}, we introduce the DP-LSSGD algorithm.
In Section~\ref{Theory-Section}, we analyze the privacy and utility guarantees of DP-LSSGD for both convex and nonconvex optimizations. We numerically verify the efficiency of DP-LSSGD in Section~\ref{Section-Experiments}. We conclude this work and point out some future directions in Section~\ref{Conclusion}.

\section{Problem Setup and Algorithm}\label{Algorithm-Section}
\subsection{Laplacian Smoothing Stochastic Gradient Descent (LSSGD)}
In this paper, we consider empirical risk minimization problem as follows. Given a training set $S=\{(\mathbf{x}_1,y_1),\ldots,(\mathbf{x}_n,y_n)\}$ drawn from some unknown but fixed distribution, we aim to find an empirical risk minimizer that minimizes the empirical risk as follows,
\begin{align}\label{Finite-Sum}
    \min_{\wb} F(\wb) :=\frac{1}{n}\sum_{i=1}^nf_i(\wb),\ \ \wb\in\RR^d,
\end{align}
where $F(\wb)$ is the empirical risk (a.k.a., training loss), $f_i(\wb) = \ell(\wb;\mathbf{x}_i,y_i)$ is the loss function of a given ML model defined on the $i$-th training example $(\mathbf{x}_i,y_i)$, and $\wb \in \RR^d$ is the model parameter we want to learn. Empirical risk minimization serves as the mathematical foundation for training many ML models that are mentioned above. The LSSGD \citep{LSGD:2018} for solving \eqref{Finite-Sum} is given by 
\begin{equation}
\label{LS-SGD}
\wb^{k+1} = \wb^k -\eta \Ab_\sigma^{-1}\bigg(\frac{1}{b}\sum_{i_k\in \mathcal{B}_k}\nabla f_{i_k}(\wb^k)\bigg),
\end{equation}
where $\eta$ is the learning rate, $\nabla f_{i_k}$ denotes the stochastic gradient of $F$ evaluated from the pair of input-output $\{\bx_{i_k}, y_{i_k}\}$, and $\mathcal{B}_k$ is a random subset of size $b$ from $[n]$. Let $\Ab_\sigma = \Ib - \sigma\mathbf{L}$ for $\sigma \geq 0$ being a constant, where $\Ib \in \RR^{d\times d}$ and $\mathbf{L} \in \RR^{d\times d}$ are the identity and the discrete one-dimensional Laplacian matrix with periodic boundary condition, respectively. Therefore,
\begin{equation}\label{eq:tri-diag}
\Ab_\sigma := 
\begin{bmatrix}
1+2\sigma   & -\sigma &  0&\dots &0& -\sigma \\
-\sigma     & 1+2\sigma & -\sigma & \dots &0&0 \\
0 & -\sigma  & 1+2\sigma & \dots & 0 & 0 \\
\dots     & \dots & \dots &\dots & \dots & \dots\\
-\sigma     &0& 0 & \dots &-\sigma & 1+2\sigma
\end{bmatrix}
\end{equation}
When $\sigma=0$, LSSGD reduces to SGD.

Note that $\Ab_\sigma$ is positive definite with condition number $1+4\sigma$ that is independent of $\Ab_\sigma$'s dimension, and LSSGD guarantees the same convergence rate as SGD in both convex and nonconvex optimization. Moreover, Laplacian smoothing (LS) 
can reduce the variance of SGD on-the-fly, and lead to better generalization in training many ML models including DNNs \citep{LSGD:2018}. For $\mathbf{v}\in \RR^d$, let $\ub:=\Ab_\sigma^{-1} \mathbf{v}$, i.e., $\mathbf{v}=\Ab_\sigma \ub$. Note $\Ab_\sigma$ is a convolution matrix, therefore, $\mathbf{v}=\Ab_\sigma \ub = \ub - \sigma\mathbf{d}*\ub$, where $\mathbf{d} = [-2, 1, 0, \cdots, 0, 1]^T$ and $*$ is the convolution operator. By the fast Fourier transform (FFT), we have
$$
\Ab_\sigma^{-1} \mathbf{v} = \ub = {\rm ifft}\left({\rm fft}(\mathbf{v})/(\mathbf{1} -\sigma \cdot {\rm fft}(\mathbf{d}))\right),
$$
where the division in the right hand side parentheses is performed in a coordinate wise way.

\subsection{DP-LSSGD}
DP ERM aims to learn a DP model, $\wb$, for the problem \eqref{Finite-Sum}. A common approach is injecting Gaussian noise into the stochastic gradient, and it resulting in the following DP-SGD
\begin{equation}
\label{DP-SGD-Nov}
\wb^{k+1} = \wb^k - \eta\bigg(\frac{1}{b}\sum_{i_k\in \mathcal{B}_k}\nabla f_{i_k}(\wb^k)+\nb\bigg),
\end{equation}
where $\nb$ is the injected Gaussian noise for DP guarantee. Note that the LS matrix $\Ab_\sigma^{-1}$ can remove the noise in $\mathbf{v}$. If we assume $\mathbf{v}$ is the initial signal, then $\Ab_\sigma^{-1}\mathbf{v}$ can be regarded as performing an approximate diffusion step on the initial noisy signal which removes the noise from $\mathbf{v}$. We will provide a detailed argument for the diffusion process in the appendix. As numerical illustrations, we consider the following two signals:
\begin{itemize}
\item 1D: $\mathbf{v}_1 = \{\sin(2i\pi/100) + 0.1\mathcal{N}(0, 1)| i = 1, 2, \cdots, 100\}$. 

\item 2D: $\mathbf{v}_2 = \{\sin(2i\pi/100)\sin(2j\pi/100) + 0.2\mathcal{N}(0, \Ib_{2\times 2})| i, j= 1, 2, \cdots, 100\}$.
\end{itemize}
We reshape $\mathbf{v}_2$ into 1D with row-major ordering and then perform LS. Figure~\ref{fig-LS-Illustration} shows that LS can remove noise efficiently. This noise removal property enables LSSGD to be more stable to the noise injected stochastic gradient, therefore improves training DP models with gradient perturbations.

\begin{figure}[!ht]
\centering
\begin{tabular}{cccc}
\includegraphics[clip, trim=0.0cm 3.5cm 0.5cm 4cm,width=0.19\columnwidth]{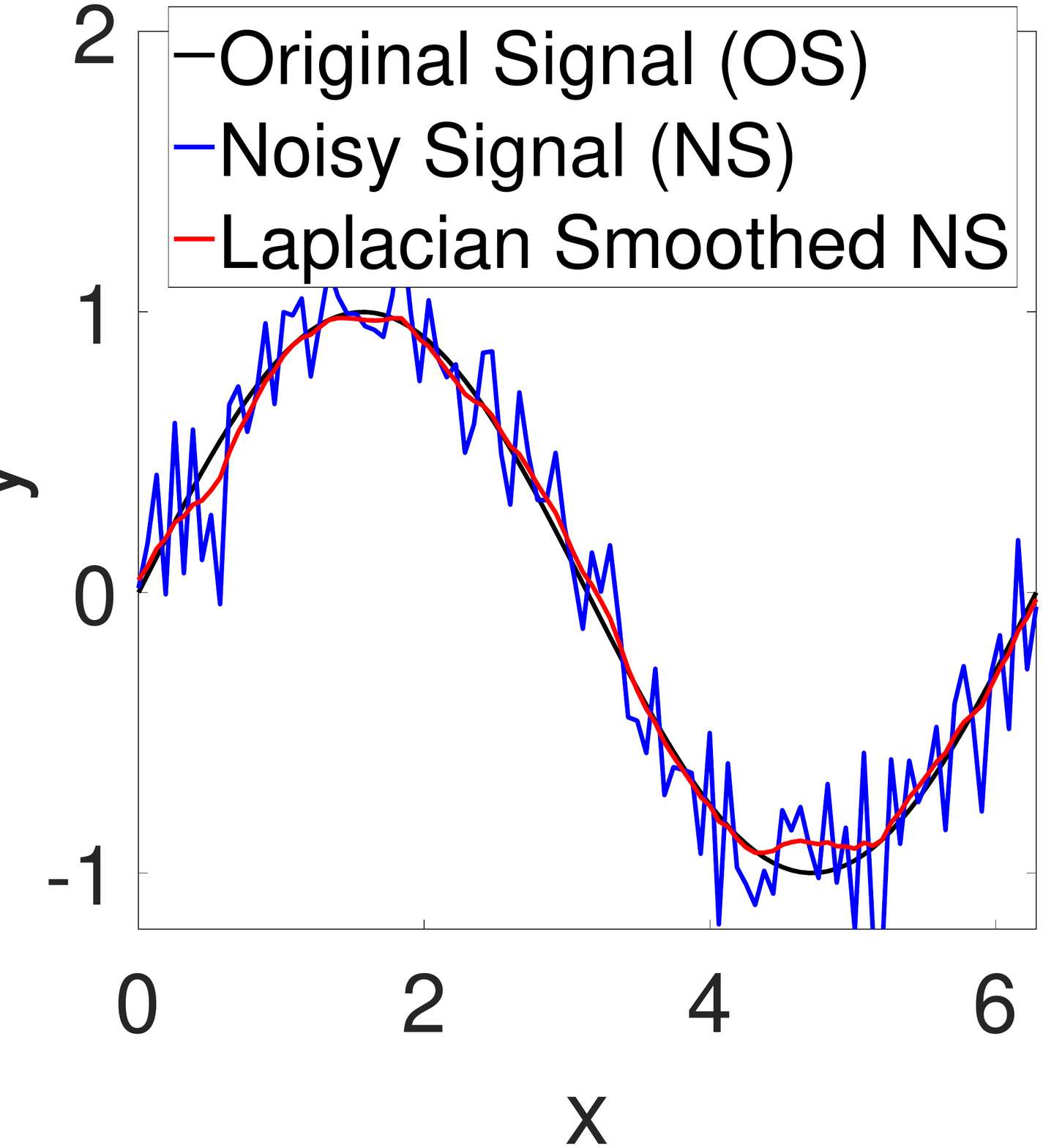}&
\includegraphics[clip, trim=2.5cm 6.5cm 2.5cm 7cm,width=0.21\columnwidth]{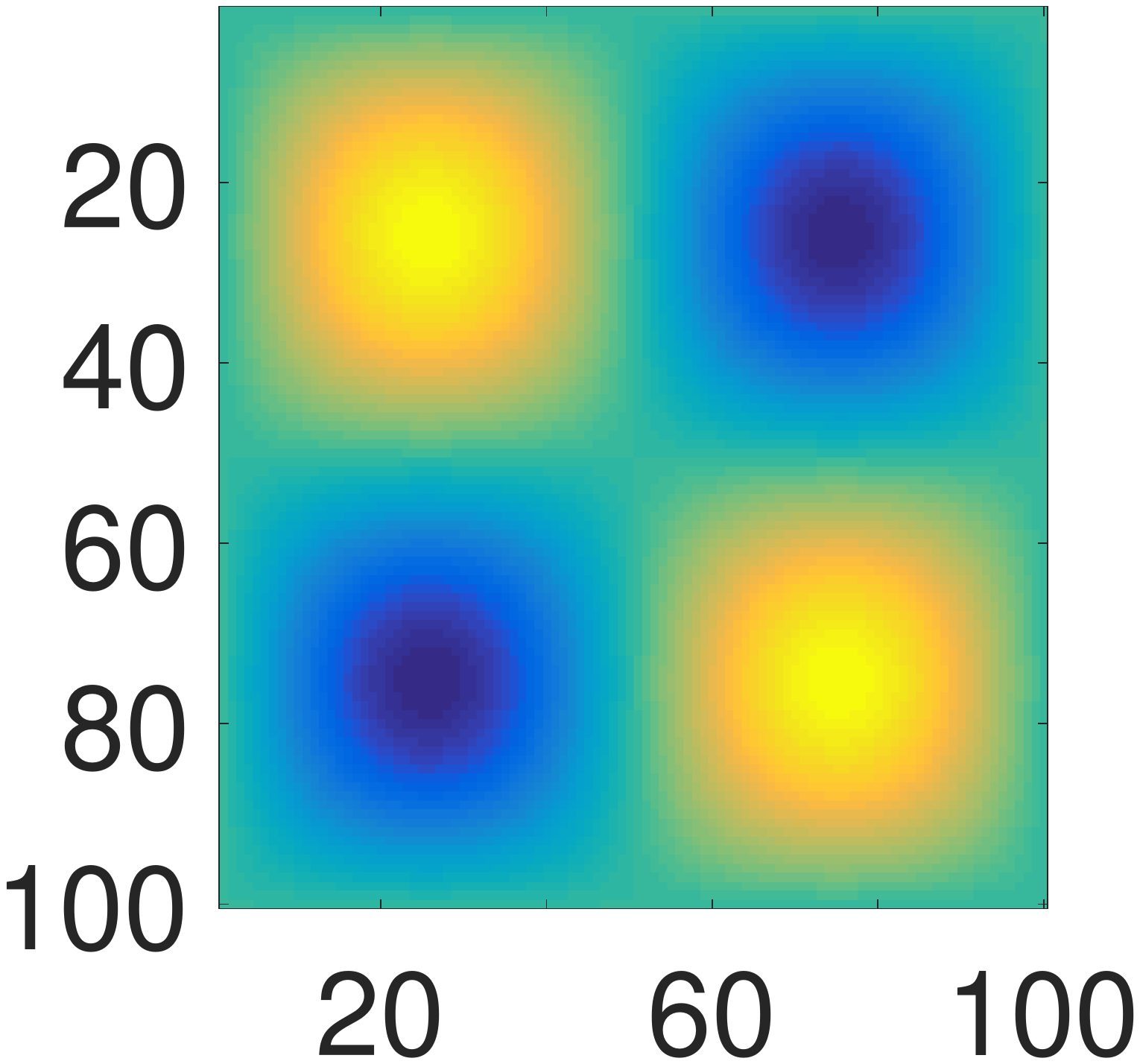}&
\includegraphics[clip, trim=2.5cm 6.5cm 2.5cm 7cm,width=0.21\columnwidth]{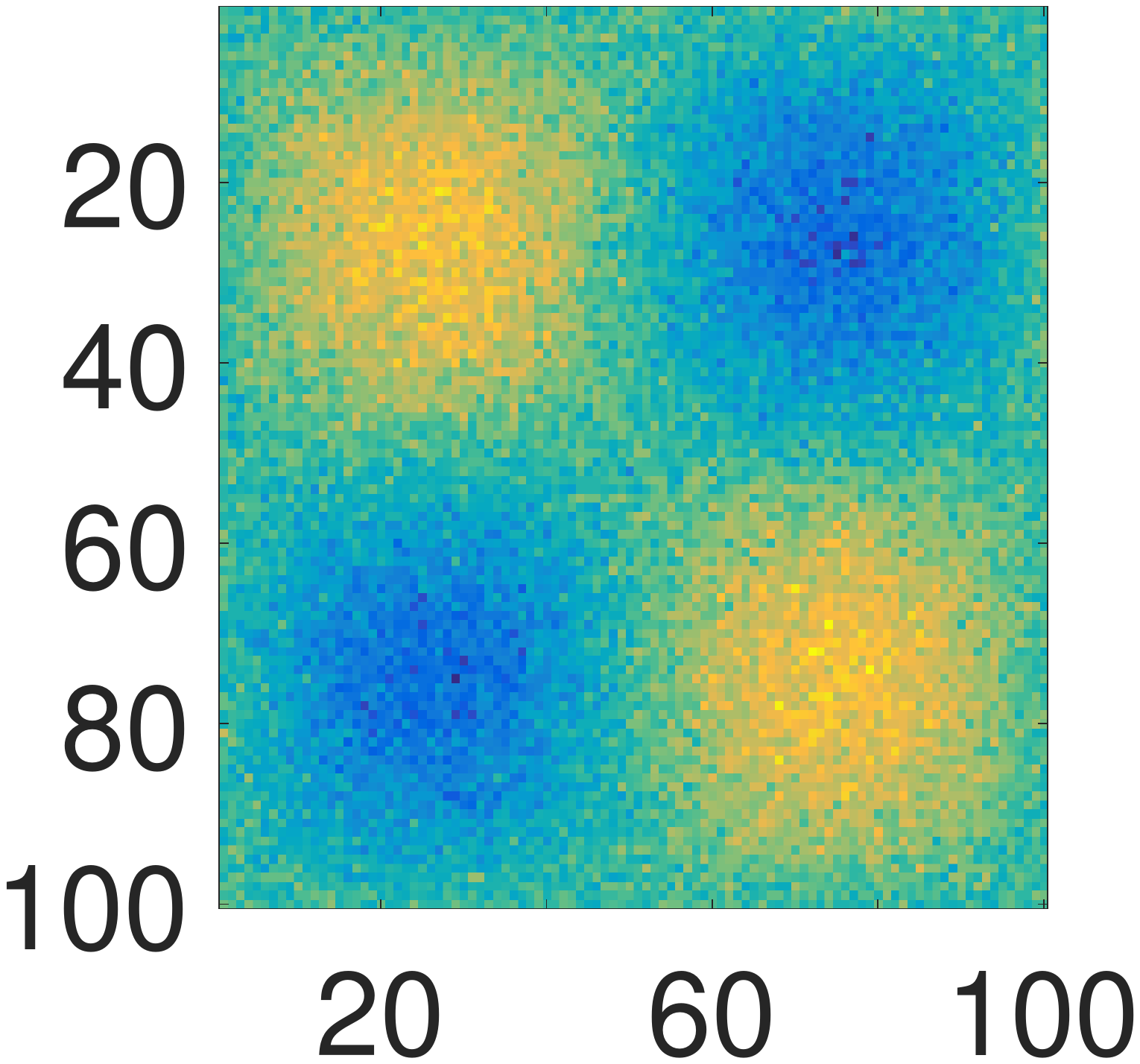}&
\includegraphics[clip, trim=2.5cm 6.5cm 2.5cm 7cm,width=0.21\columnwidth]{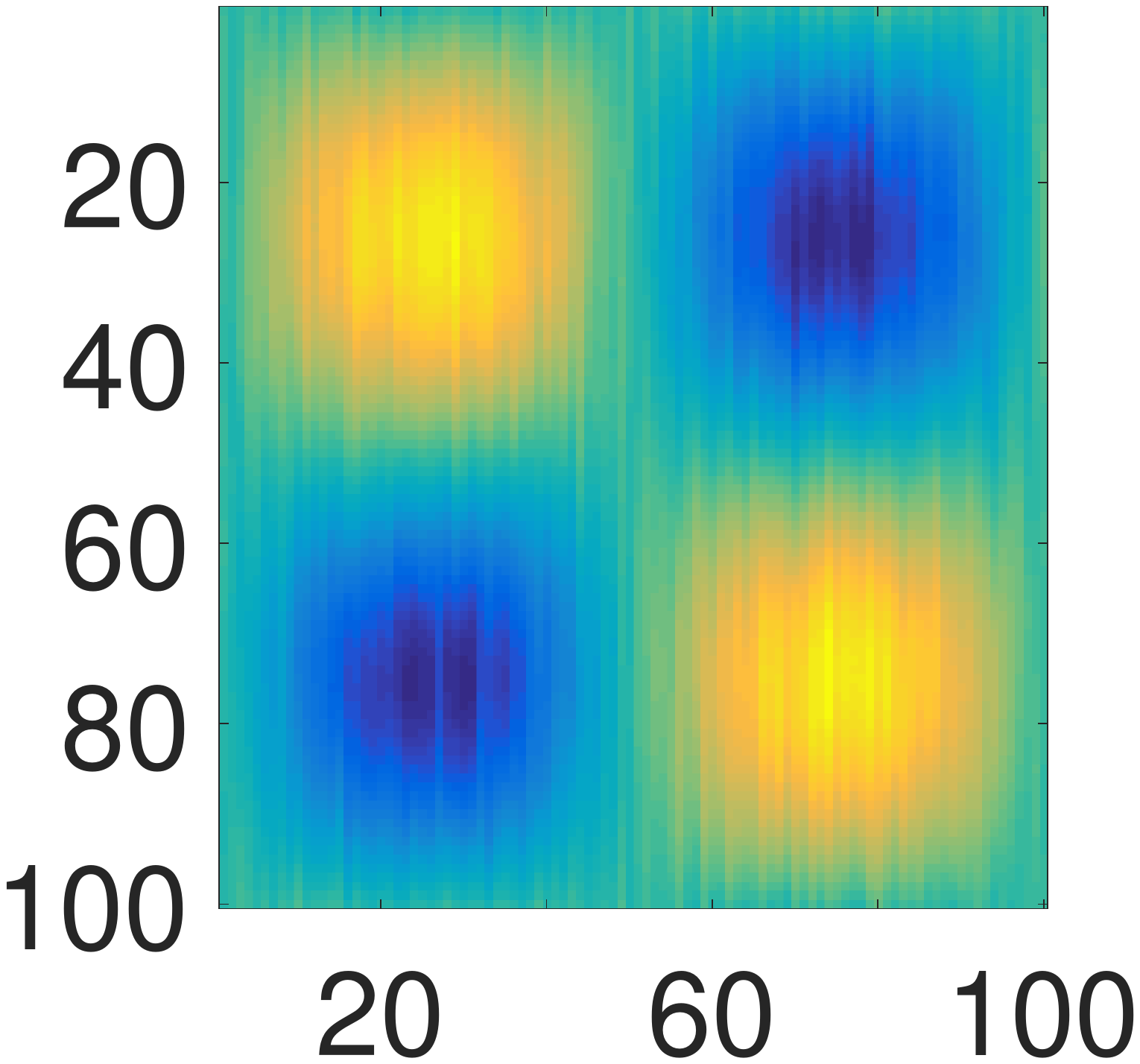}\\
(a)  & (b)  & (c)  & (d)  \\
\end{tabular}
\caption{Illustration of LS ($\sigma=10$ for $\mathbf{v}_1$ and $\sigma=100$ for $\mathbf{v}_2$). 
(a): 
1D signal 
sampled uniformly from $\sin(x)$ for $x\in[0, 2\pi]$. (b), (c), (d): 2D original, noisy, and Laplacian Smoothed noisy signals 
sampled uniformly from $\sin(x)\sin(y)$ for $(x, y) \in [0, 2\pi]\times [0, 2\pi]$.}
\label{fig-LS-Illustration}
\end{figure}

We propose the following DP-LSSGD for solving \eqref{Finite-Sum} with DP guarantee
\begin{equation}
\label{SGD-1}
\wb^{k+1} = \wb^k - \eta\Ab_\sigma^{-1}\bigg(\frac{1}{b}\sum_{i_k\in \mathcal{B}_k}\nabla f_{i_k}(\wb^k)+\nb\bigg).
\end{equation}
In this scheme, we first inject the noise $\nb$ to the stochastic gradient $\nabla f_{i_k}(\wb^k)$, and then 
apply the LS operator $\Ab_\sigma^{-1}$ to denoise the noisy stochastic gradient, $\nabla f_{i_k}(\wb^k) + \nb$, on-the-fly.
We assume that each component function $f_i$ in \eqref{Finite-Sum} is $G$-Lipschitz. The DP-LSSGD 
for finite-sum optimization is summarized in Algorithm~\ref{DP-LSSGD-Pseudocode}. Compared with LSSGD, the main difference of DP-LSSGD lies in injecting Gaussian noise into the stochastic gradient, before applying the Laplacian smoothing, to guarantee the DP.


\begin{algorithm}[t]
\caption{DP-LSSGD}\label{DP-LSSGD-Pseudocode}
\begin{algorithmic}
\State \textbf{Input: } $f_i(\wb)$ is $G$-Lipschitz for $i=1, 2, \cdots, n$. \\
$\wb^0$: initial guess of $\wb$, $(\epsilon, \delta)$: the privacy budget, $\eta$: the step size, $T$: the total number of iterations.
\State \textbf{Output: } $(\epsilon, \delta)$-differentially private classifier $\wb_{\rm priv}$.
\For {$k=0, 1, \cdots, T-1$}
\State $\wb^{k+1} = \wb^k - \eta\Ab_\sigma^{-1}\left(\frac{1}{b}\sum_{i_k\in \mathcal{B}_k}\nabla f_{i_k}(\wb^k)+\nb\right)$, where $\mathbf{n}\sim \mathcal{N}(\mathbf{0}, \nu^2\mathbf{I})$ and $\nu$ is defined in Theorem~\ref{Theorem-privacy-guarantee}, and $\mathcal{B}_k\subset [n]$.
\EndFor
\Return $\wb^T$
\end{algorithmic}
\end{algorithm}

\section{Main Theory}\label{Theory-Section}
In this section, we present the privacy and utility guarantees for DP-LSSGD. The technical proofs are provided in the appendix.
\begin{definition}[$(\epsilon, \delta)$-DP] \label{DP-Def-Dwork} (\cite{Dwork:2006-Basic}) 
A randomized mechanism $\cM:\cS^N\rightarrow\cR$ satisfies $(\epsilon,\delta)$-DP if for any two adjacent datasets $S,S'\in \cS^N$ differing by one element, and any output subset $O\subseteq \cR$, it holds that
\begin{align*}
\PP[\cM(S)\in O]\leq e^\epsilon\cdot \PP[\cM(S')\in O]+\delta.
\end{align*}
\end{definition}

\begin{theorem}[Privacy Guarantee]\label{Theorem-privacy-guarantee}
	Suppose that each component function $f_i$ is $G$-Lipschitz. Given the total number of iterations $T$, for any $\delta>0$ and privacy budget $\epsilon$, DP-LSSGD, with injected Gaussian noise $\mathcal{N}(0, \nu^2)$ for each coordinate, satisfies $(\epsilon,\delta)$-DP with $\nu^2=20T\alpha G^2/(\mu n^2\epsilon)$, where $\alpha=\log(1/\delta)/\big((1-\mu)\epsilon\big)+1$, if there exits $\mu\in(0,1)$ such that  $\alpha \leq \log\big(\mu n^3\epsilon/(5b^3T\alpha+\mu b n^2\epsilon)\big)$ and $5b^2T\alpha/(\mu n^2\epsilon)\geq 1.5$.
\end{theorem}

\begin{remark}
It is straightforward to show that the noise in Theorem~\ref{Theorem-privacy-guarantee} is in fact also tight to guarantee the $(\epsilon, \delta)$-DP for DP-SGD. . We will omit the dependence of $\mu$ in our results in the rest of the paper since $\mu$ is a constant.
\end{remark}

For convex ERM, DP-LSSGD guarantees the following utility in terms of the gap between the ergodic average of the points along the DP-LSSGD path and the optimal solution $\wb^*$.

\begin{theorem}[Utility Guarantee for convex optimization]\label{Convex-Utility}
Suppose $F$ is convex and each component function $f_i$ is $G$-Lipschitz. Given $\epsilon,\delta>0$, under the same conditions of Theorem \ref{Theorem-privacy-guarantee} on $\nu^2,\alpha$, if we choose $\eta_k = 1/\sqrt{T}$ and
$T=C_1(D_\sigma+G^2/b)n^2\epsilon^2/\big(dG^2\log(1/\delta)\big)$, where $D_\sigma=\|\wb^0-\wb^*\|_{\Ab_\sigma}^2$ and $\wb^*$ is the global minimizer of $F$, the DP-LSSGD output $\tilde \wb=\sum_{k=0}^{T-1}\eta_k/\big(\sum_{i=0}^{T-1}\eta_i\big)\wb^k$ satisfies the following utility
\begin{align*}
    \EE\big(F(\tilde \wb)-F(\wb^*)\big)\leq \frac{C_2G\sqrt{6\gamma (D_\sigma+G^2/b)d\log(1/\delta)}}{n\epsilon},
\end{align*}
where $\gamma = 1/d\sum_{i=1}^d 1/[1+2\sigma - 2\sigma \cos(2\pi i/d)]$, $C_1,C_2$ are universal constants.
\end{theorem}

\begin{proposition}\label{Gamma-Value}
In Theorem~\ref{Convex-Utility}, $\gamma = \frac{1+\omega^d}{(1-\omega^d)\sqrt{4\sigma+1}},$
where $\omega = \frac{2\sigma+1 - \sqrt{4\sigma+1}}{2\sigma} < 1.$ That is, $\gamma$ converge to $0$ almost exponentially as the dimension, $d$, increases. 
\end{proposition}

\begin{remark}
In the above utility bound for convex optimization, for different $\sigma$ ($\sigma=0$ corresponds to DP-SGD), the only difference lies in the term $\gamma(D_\sigma + G^2)$.
The first part $\gamma D_\sigma$ depends on the gap between initialization $\wb^0$ and the optimal solution $\wb^*$. The second part $\gamma G^2$ decrease monotonically as $\sigma$ increases. $\sigma$ should be selected to get an optimal trade-off between these two parts. Based on our test on multi-class logistic regression for MNIST classification, $\sigma \neq 0$ always outperforms the case when $\sigma=0$.
\end{remark}

For nonconvex ERM,  DP-LSSGD has the following utility bound measured in gradient norm.

\begin{theorem}[Utility Guarantee for nonconvex optimization]\label{Nonconvex-Utility}
Suppose that $F$ is nonconvex and each component function $f_i$ is $G$-Lipschitz and has $L$-Lipschitz continuous gradient. Given $\epsilon,\delta>0$, under the same conditions of Theorem \ref{Theorem-privacy-guarantee} on $\nu^2,\alpha$, if we choose $\eta=1/\sqrt{T}$ and $T=C_1(D_F+LG^2/b)n^2\epsilon^2/\big(dLG^2\log(1/\delta)\big)$, where $D_F=F(\wb^0)-F(\wb^*)$ with $\wb^*$ being the global minimum of $F$, then the DP-LSSGD output $\tilde \wb=\sum_{k=0}^{T-1}\wb^k/T$ satisfies the following utility
\begin{align*}
    \EE\|\nabla F(\tilde \wb)\|_{\Ab_{\sigma}^{-1}}^2\leq C_2\frac{ G\sqrt{\beta dL(2D_F+LG^2/b)\log(1/\delta)}}{n\epsilon},
\end{align*}
where $\beta = 1/d\sum_{i=1}^d 1/[1+2\sigma - 2\sigma \cos(2\pi i/d)]^2$, $C_1,C_2$ are universal constants.
\end{theorem}

\begin{proposition}\label{Beta-Value}
In Theorem~\ref{Nonconvex-Utility}, 
$
\beta = \frac{2\omega^{2d+1}-\xi\omega^{2d}+2\xi d\omega^d-2\omega+\xi}{\sigma^2\xi^3(1-\omega^d)^2},
$
where
$
\omega = \frac{2\sigma+1 - \sqrt{4\sigma+1}}{2\sigma}
$
and
$
\xi = -\frac{\sqrt{1+4\sigma}}{\sigma}.
$ 
Therefore, $\beta\in (0, 1)$.
\end{proposition}

It is worth noting that if we use the $\ell_2$-norm instead of the induced norm, we have the following utility guarantee
$$
{\footnotesize
\EE\|\nabla F(\tilde \wb)\|_2^2 \leq \frac{\EE\|\nabla F(\tilde \wb)\|_{\Ab_{\sigma}^{-1}}^2}{\lambda_{\min} (\Ab_{\sigma}^{-1})}
\leq (1+4\sigma)\EE\|\nabla F(\tilde \wb)\|_{\Ab_{\sigma}^{-1}}^2
\leq 4 \zeta \frac{ G\sqrt{6dL(2D_F+LG^2)\log(1/\delta)}}{n\epsilon}
}
$$
where $\zeta = \sqrt{\frac{1}{d}\sum_{i=1}^d \frac{(1+4\sigma)^2}{(1+2\sigma - 2\sigma \cos(2\pi i/d))^2 } } > 1$. In the $\ell_2$-norm, DP-LSSGD has a bigger utility upper bound than DP-SGD (set $\sigma=0$ in $\zeta$). However, this does not mean that DP-LSSGD has worse performance. To see this point, let us consider the following simple nonconvex function
\begin{equation}
\label{Nonconvex-Example}
f(x, y) = \begin{cases}
        \frac{x^2}{4} + y^2, & \text{for } \frac{x^2}{4} + y^2\leq 1\\
        \sin\left(\frac{\pi}{2} \left(\frac{x^2}{4} + y^2\right)\right), & \text{for } \frac{x^2}{4} + y^2 > 1.
        \end{cases}
\end{equation}
For two points $\mathbf{a}_1 = (2, 0)$ and $\mathbf{a}_2 = (1, \sqrt{3}/2)$, the distance to the local minima $\mathbf{a}^* = (0, 0)$ are $2$ and $\sqrt{7}/2$, while $\|\nabla f(\mathbf{a}_1)\|_2 = 1$ and $\|\nabla f(\mathbf{a}_2)\|_2 = \sqrt{13}/2$.
So $\mathbf{a}_2$ is closer to the local minima $\mathbf{a}^*$ than $\mathbf{a}_1$ while its gradient has a larger $\ell_2$-norm. 


\section{Experiments}\label{Section-Experiments}
In this section, we verify the efficiency of DP-LSSGD in training multi-class logistic regression and CNNs for MNIST and CIFAR10 classification. 
We use 
$\mathbf{v}\leftarrow \mathbf{v}/\max{\left(1, \|\mathbf{v}\|_2/C \right)}$ \citep{Abadi:2016} to clip the gradient $\ell_2$-norms of the CNNs to $C$. The gradient clipping guarantee the Lipschitz condition for the objective functions. We train all the models below with $(\epsilon, 10^{-5})$-DP guarantee for different $\epsilon$. For Logistic regression we use the privacy budget given by Theorem~\ref{Theorem-privacy-guarantee}, and for CNNs we use the privacy budget in the Tensorflow privacy \citep{Tensorflow:Privacy}. We checked that these two privacy budgets are consistent.



\subsection{Logistic Regression for MNIST Classification}
We ran $50$ epochs of DP-LSSGD with learning rate scheduled as $1/t$ with $t$ being the index of the iteration to train the $\ell_2$-regularized (regularization constant $10^{-4}$) multi-class logistic regression.
We split the training data into $50$K/$10$K with batch size $128$ for cross-validation. 
We plot the evolution of training and validation loss over iterations for privacy budgets $(0.2, 10^{-5})$ and $(0.1, 10^{-5})$ in Figure~\ref{fig:Training-Testing-Loss-Softmax}. We see that the training loss curve of DP-SGD ($\sigma=0$) is much higher and more oscillatory (log-scale on the $y$-axis) than that of DP-LSSGD ($\sigma=1, 3$).
Also, the validation loss of the model trained by DP-LSSGD decays faster and has a much smaller loss value than that of the model trained by DP-SGD. Moreover, when the privacy guarantee gets stronger, the utility improvement by DP-LSSGD becomes more significant.

\begin{figure}[!ht]
\centering
\begin{tabular}{cc}
\includegraphics[clip, trim=0.5cm 4.5cm 1.8cm 5cm,width=0.35\columnwidth]{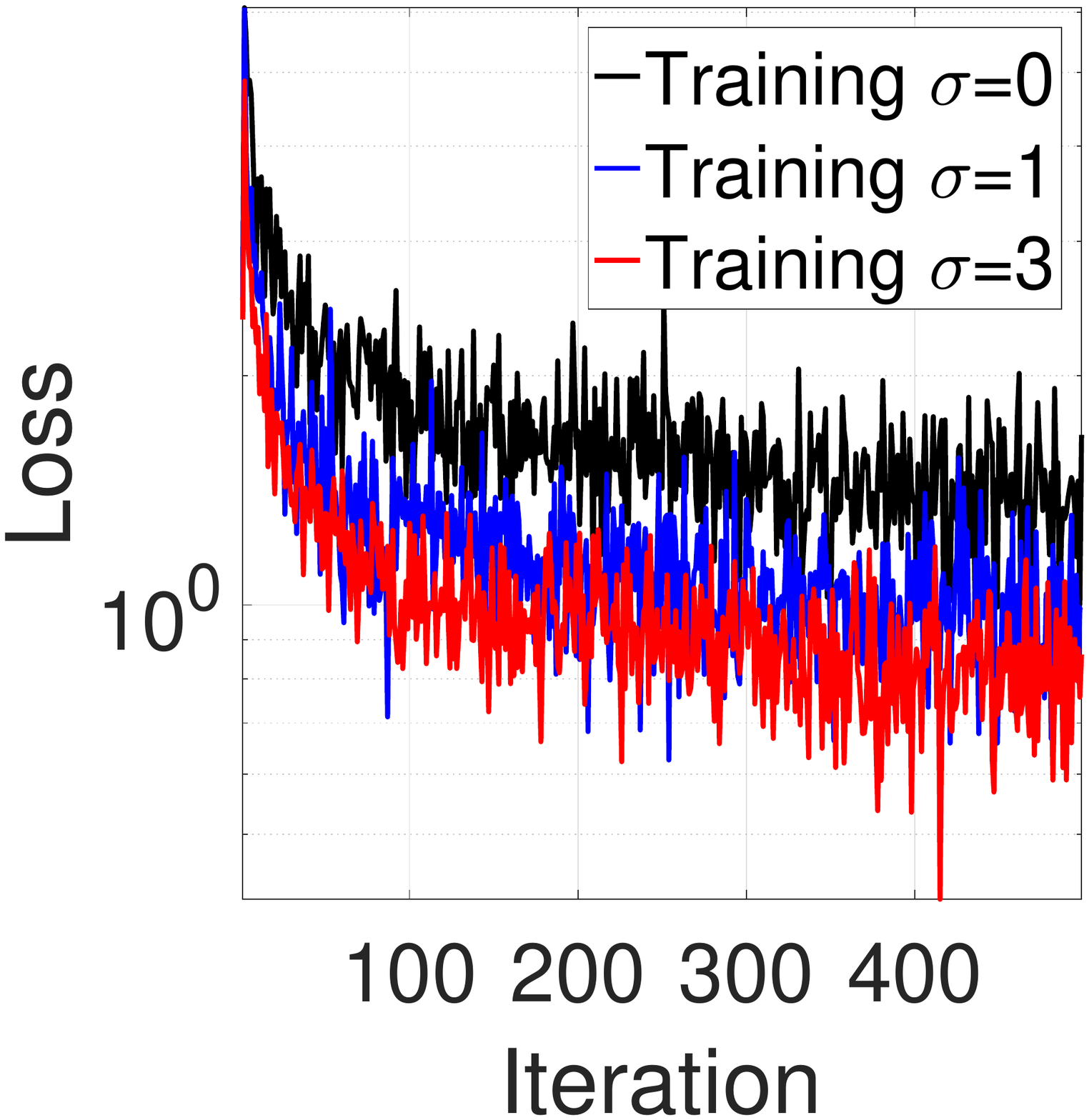}&
\includegraphics[clip, trim=0.5cm 4.5cm 1.8cm 5cm,width=0.35\columnwidth]{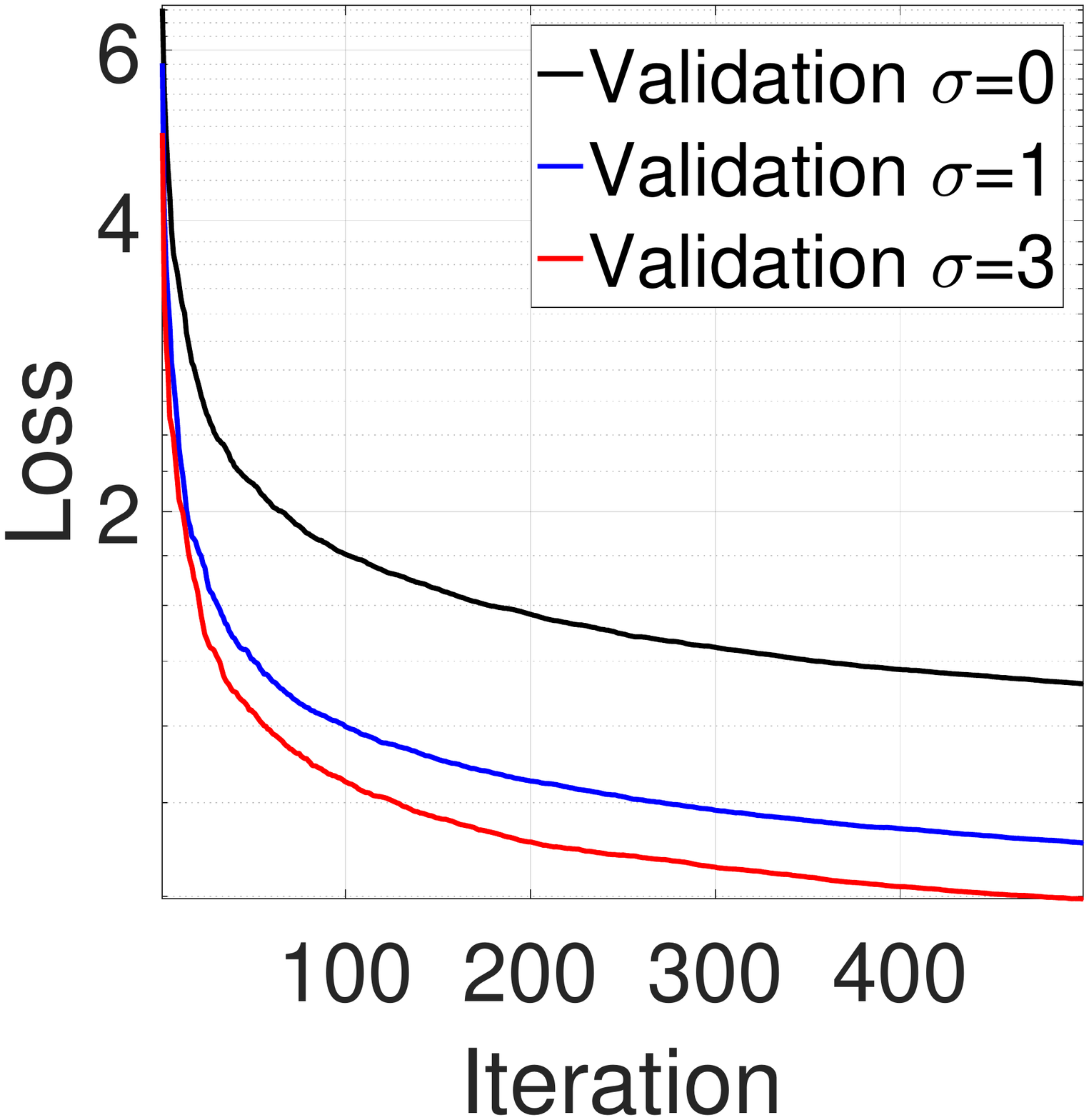}\\
\includegraphics[clip, trim=0.5cm 4.5cm 1.8cm 5cm,width=0.35\columnwidth]{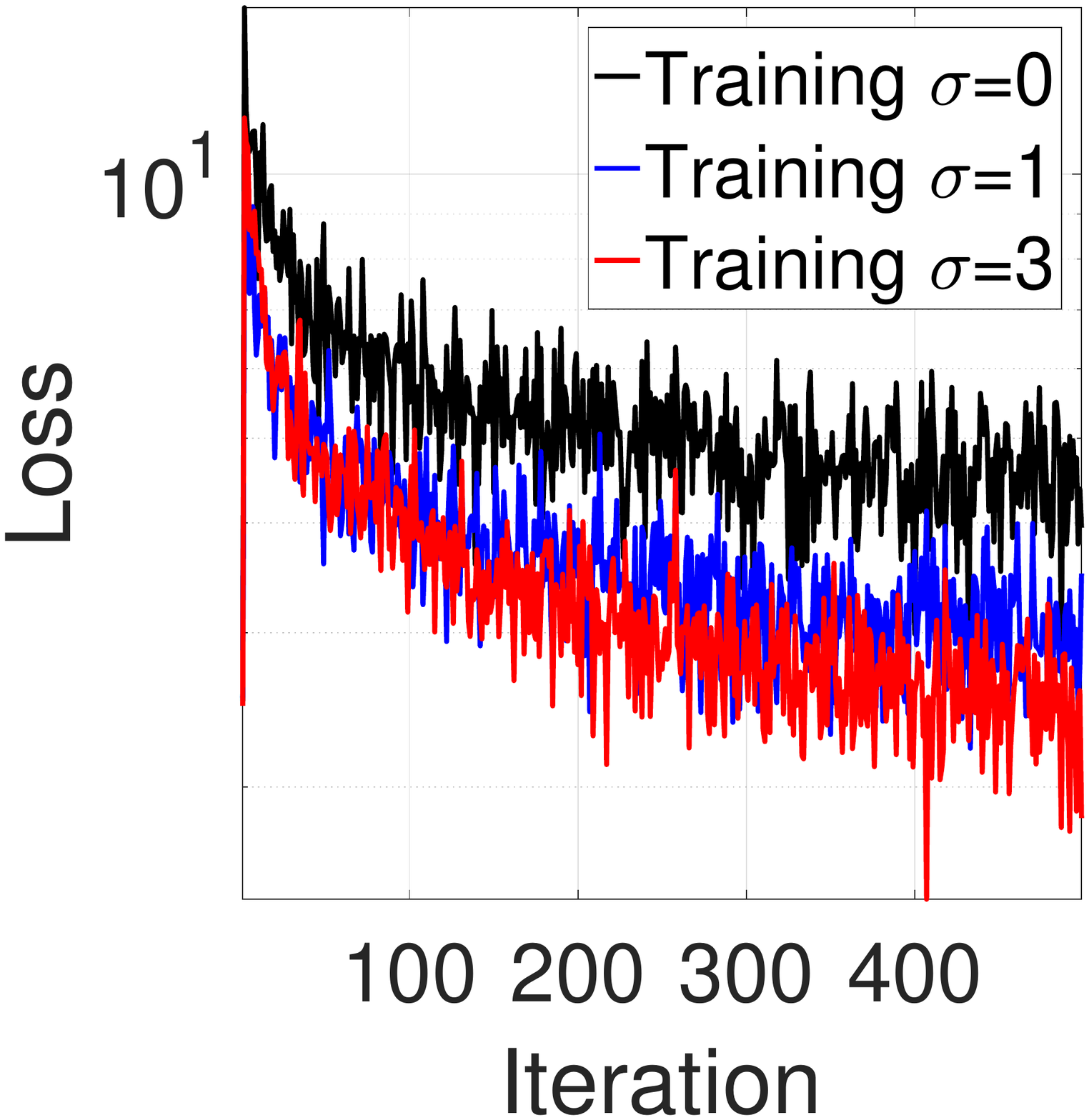}&
\includegraphics[clip, trim=0.5cm 4.5cm 1.8cm 5cm,width=0.35\columnwidth]{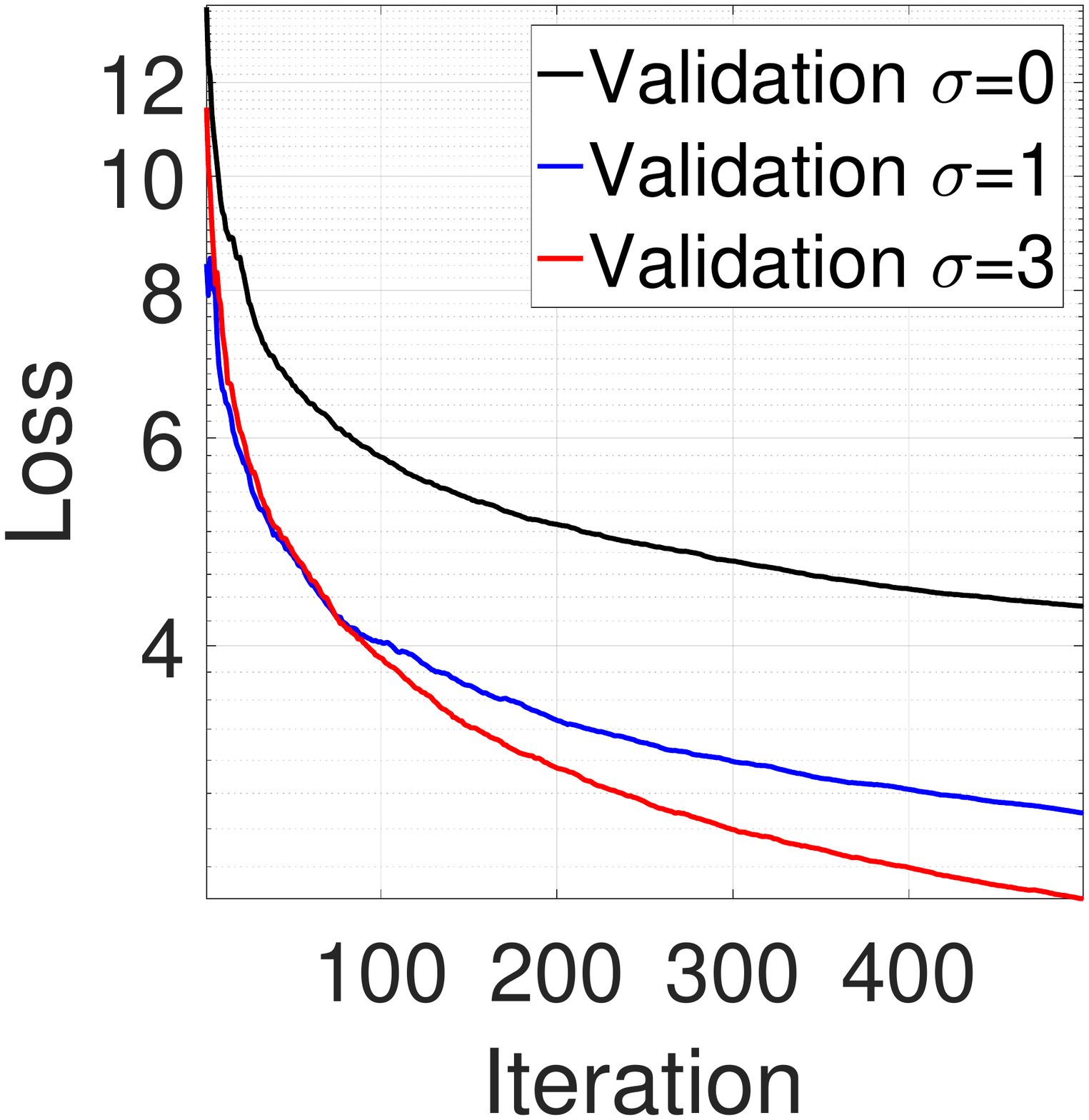}\\
\end{tabular}
\caption{Training and validation losses of the multi-class logistic regression by DP-LSSGD. (a) and (b): training and validation curves with $(0.2, 10^{-5})$-DP guarantee; (c) and (d): training and validation curves with $(0.1, 10^{-5})$-DP guarantee. (Average over 5 runs)
}
\label{fig:Training-Testing-Loss-Softmax}
\end{figure}
Next, consider the testing accuracy of the multi-class logistic regression trained with $(\epsilon, 10^{-5})$-DP guarantee by DP-LSSGD includes $\sigma=0$, i.e., DP-SGD. We list the test accuracy of logistic regression trained in different settings in Table~\ref{Table1:Gaussian:Softmax:WeightDecay}. These results reveal that DP-LSSGD with $\sigma=1, 2, 3$ can improve the accuracy of the trained private model and also reduce the variance, especially when the privacy guarantee is very strong, e.g., $(0.1, 10^{-5})$.

\begin{table}[!ht]
\centering
\fontsize{8.5}{8.5}\selectfont
\begin{threeparttable}
\caption{Testing accuracy of the multi-class logistic regression trained by DP-LSSGD with $(\epsilon, \delta=10^{-5})$-DP guarantee and different LS parameter $\sigma$. Unit: \%. (5 runs)
}\label{Table1:Gaussian:Softmax:WeightDecay}
\begin{tabular}{ccccccc}
\toprule[1.0pt]
\ \ \ \ \ \ \ \ \ \  $\epsilon$\ \ \ \ \ \ \ \ \ \    &\ \ \ \ \ \ \ \ \ \  0.30\ \ \ \ \ \ \ \ \   &\ \ \ \ \ \ \   0.25\ \ \ \ \ \ \   &\ \ \ \ \ \ \ 0.20\ \ \ \ \ \ \   &\ \ \ \ \ \ \   0.15\ \ \ \ \ \ \   &\ \ \ \ \ \ \  0.10\ \ \ \ \ \ \     \cr
\midrule[0.8pt]
$\sigma=0$  & 81.74 $\pm$ 0.96 & 81.45 $\pm$ 1.59 & 78.92 $\pm$ 1.14 & 77.03 $\pm$ 0.69& 73.49 $\pm$ 1.60\cr
$\sigma=1$  & 84.21 $\pm$ 0.51 & 83.27 $\pm$ 0.35 & 81.56 $\pm$ 0.79 & 79.46 $\pm$ 1.33 &76.29 $\pm$ 0.53\cr
$\sigma=2$  & 84.23 $\pm$ 0.65 & {\bf 83.65 $\pm$ 0.76} & 82.15 $\pm$ 0.59 & 80.77 $\pm$ 1.26 &  76.31 $\pm$ 0.93\cr
$\sigma=3$  & {\bf 85.11 $\pm$ 0.45} & 82.97 $\pm$ 0.48 & {\bf 82.22 $\pm$ 0.28} & {\bf 80.81 $\pm$ 1.03} & {\bf 77.13 $\pm$ 0.77}\cr
\bottomrule[1.0pt]
\end{tabular}
\end{threeparttable}
\end{table}

\begin{figure}[!ht]
\centering
\begin{tabular}{cc}
\includegraphics[clip, trim=0.2cm 4.0cm 1.2cm 5cm,width=0.42\columnwidth]{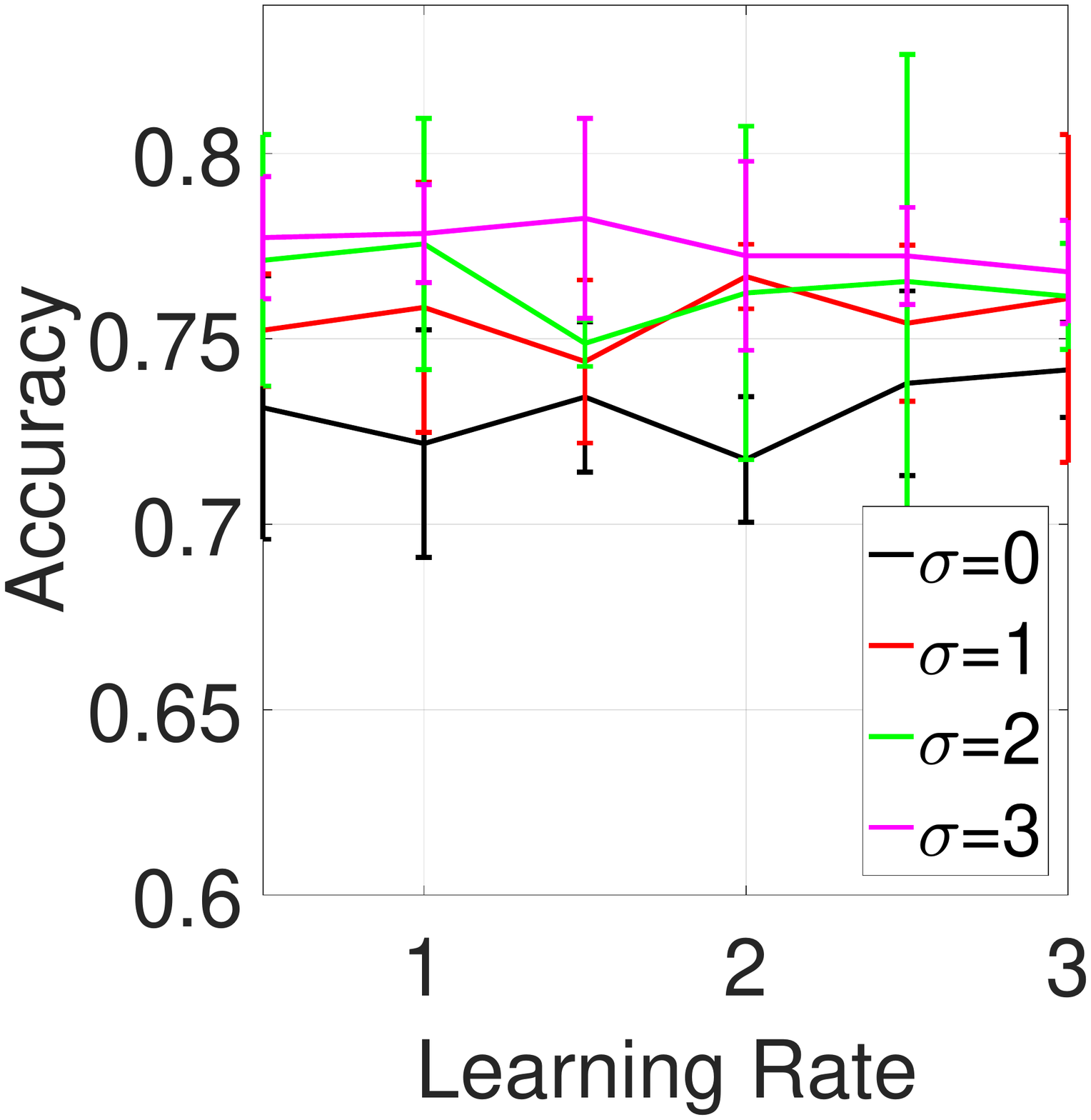}&
\includegraphics[clip, trim=0.2cm 4.0cm 1.2cm 5cm,width=0.42\columnwidth]{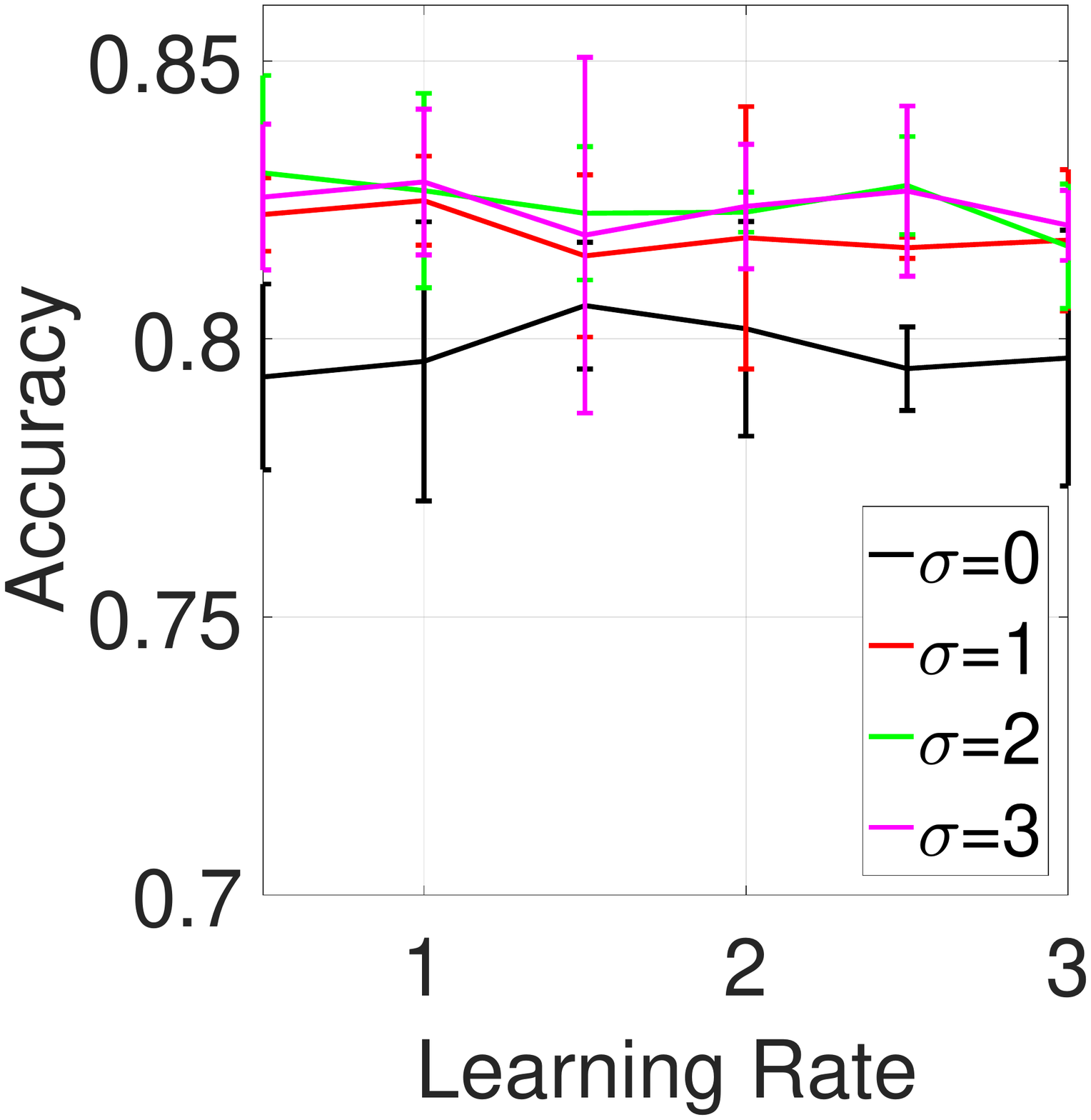}\\
\end{tabular}
\caption{Accuracy of the logistic regression on MNIST when different learning rates are used to train the model. Left: $(0.1, 10^{-5})$-DP; Right: $(0.2, 10^{-5})$-DP.}
\label{fig:LearningRate-Softmax}
\end{figure}

\subsubsection{The Effects of Step Size}
We know that the step size in DP-SGD/DP-LSSGD may affect the accuracy of the trained private models. We try different step size scheduling of the form $\{a/t| a= 0.5, 1.0, 1.5, 2.0, 2.5, 3.0\}$, where $t$ is again the index of iteration, and all the other hyper-parameters are used the same as before. Figure.~\ref{fig:LearningRate-Softmax} plots the test accuracy of the logistic regression model trained with different learning rate scheduling and different  
privacy budget. We see that the private logistic regression model trained by DP-LSSGD always outperforms DP-SGD.


\subsection{CNN for MNIST and CIFAR10 Classification}
In this subsection, we consider training a small CNN \footnote{\url{github.com/tensorflow/privacy/blob/master/tutorials/mnist_dpsgd_tutorial.py}} with DP-guarantee for MNIST classification. We implement DP-LSSGD and DP-LSAdam \citep{kingma2014adam} (simply replace the noisy gradient in DP-Adam in the Tensorflow privacy with the Laplacian smoothed surrogate) into the Tensorflow privacy framework \citep{Tensorflow:Privacy}. We use the default learning rate $0.15$ for DP-(LS)SGD and $0.001$ for DP-(LS)Adam and decay them by a factor of $10$ at the $10$K-th iteration, norm clipping ($1$), batch size ($256$), and micro-batches ($256$). We vary the noise multiplier (NM), and larger NM guarantees stronger DP. As shown in Figure~\ref{fig-CNN-MNIST-Acc-Eps}, the 
privacy budget increases at exactly the same speed (dashed red line) for four optimization algorithms. When the NM is large, i.e., DP-guarantee is strong, DP-SGD performs very well in the initial period. However, after a few epochs, the validation accuracy gets highly oscillatory and decays. DP-LSSGD can mitigate the training instability issue of DP-SGD. DP-Adam outperforms DP-LSSGD, and DP-LSAdam can further improve validation accuracy on top of DP-Adam.

\begin{figure}[!ht]
\centering
\begin{tabular}{cc}
\includegraphics[clip, trim=0.1cm 5cm 0.1cm 6.5cm,width=0.35\columnwidth]{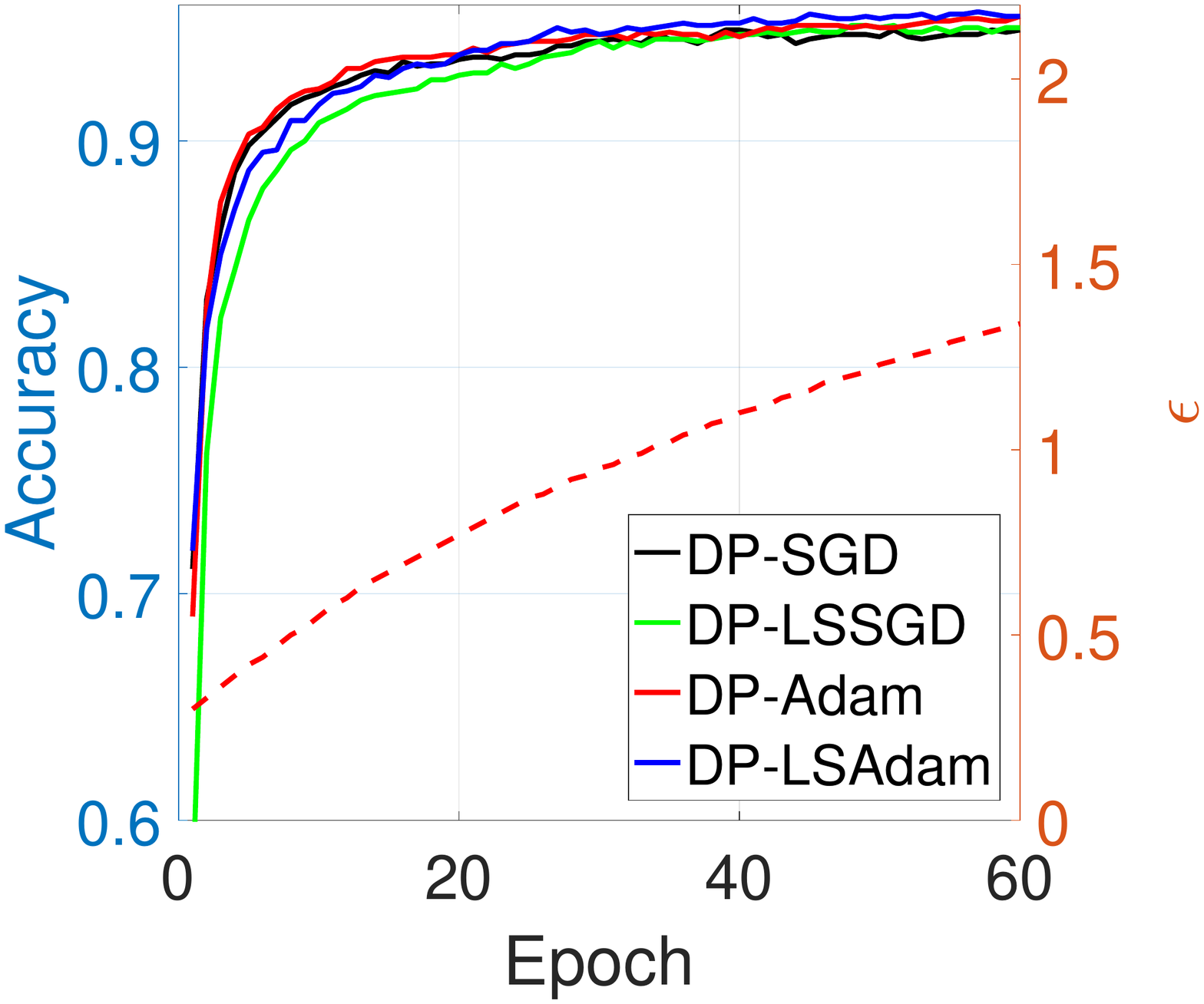}&
\includegraphics[clip, trim=0.1cm 5cm 0.1cm 6.5cm,width=0.35\columnwidth]{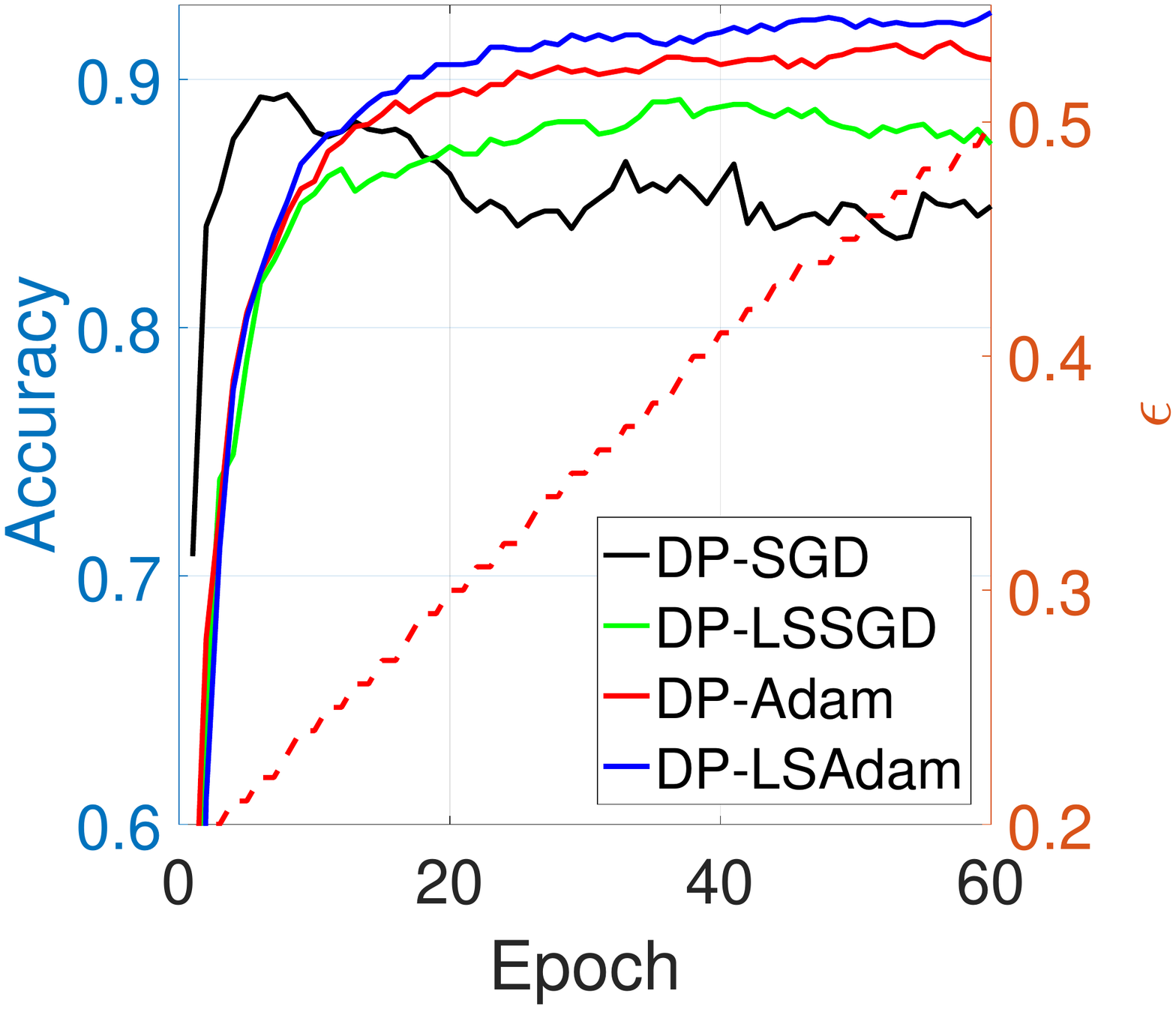}\\
{\footnotesize NM $2.0$ (LS: $\sigma=0.2$)}  & {\footnotesize NM $5.0$ (LS: $\sigma=1.0$)} \\
\includegraphics[clip, trim=0.1cm 5cm 0.1cm 6.5cm,width=0.35\columnwidth]{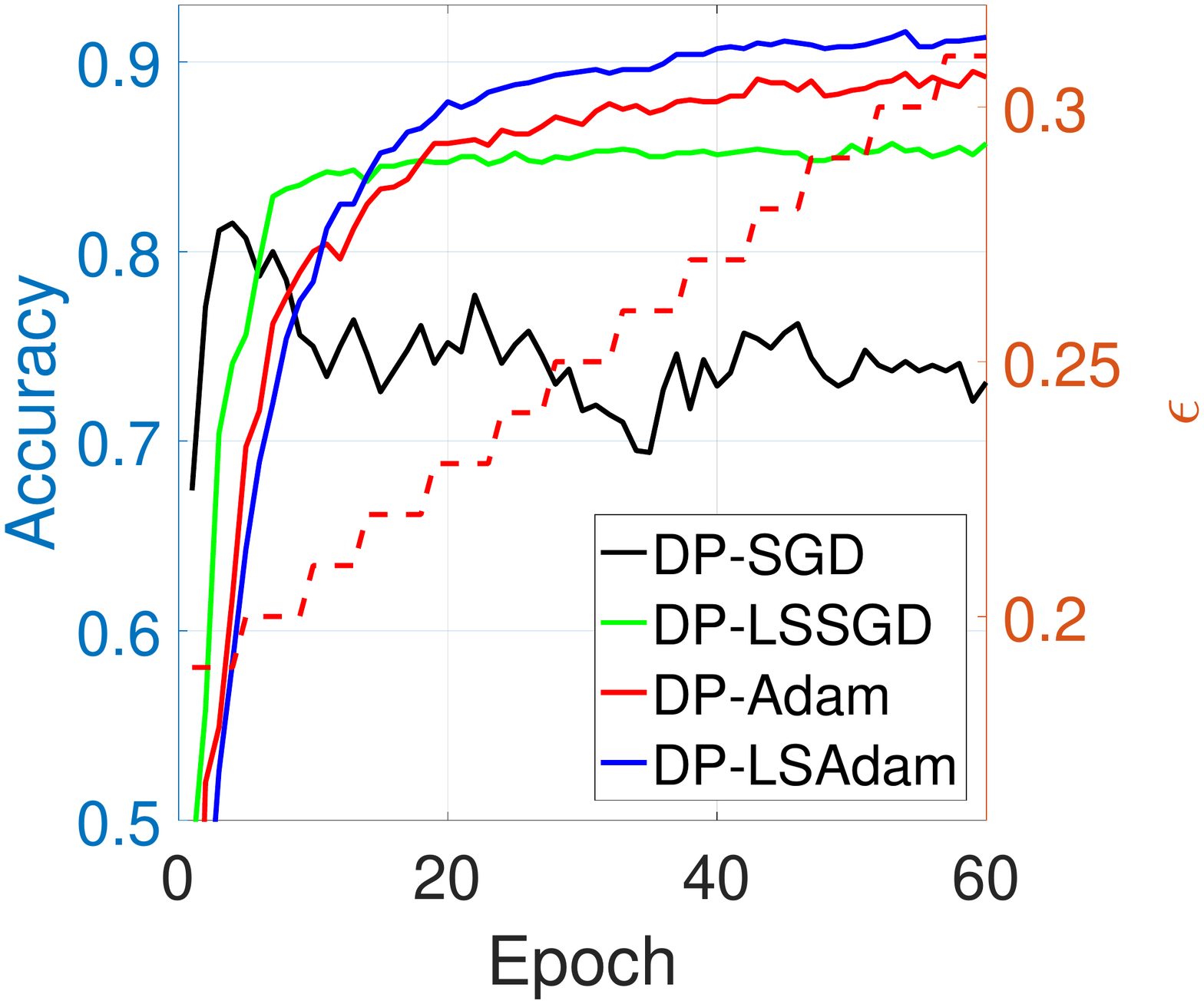}&
\includegraphics[clip, trim=0.1cm 5cm 0.1cm 6.5cm,width=0.35\columnwidth]{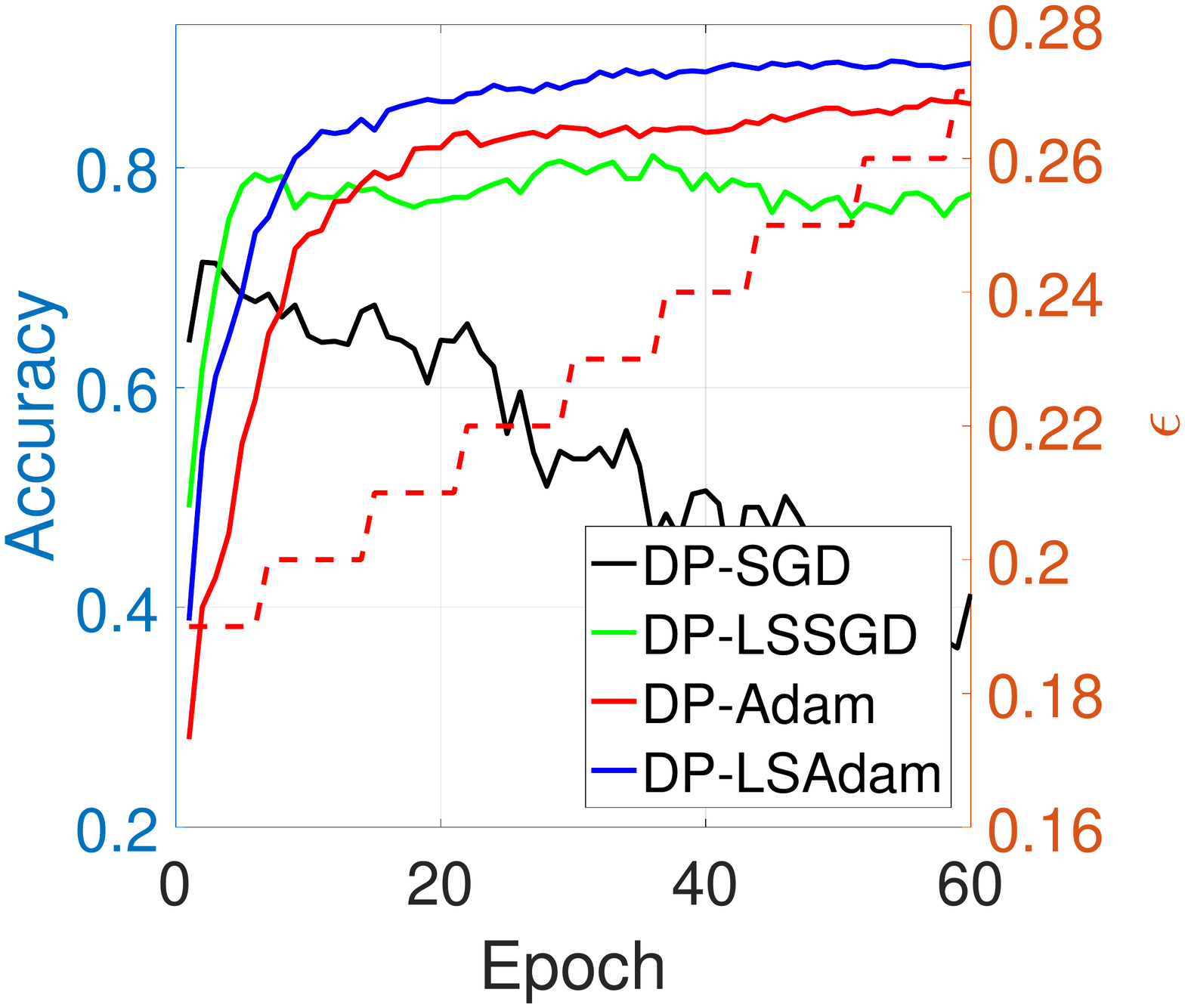}\\
 {\footnotesize NM $8.0$ (LS: $\sigma=1.0$)}  & {\footnotesize NM $10.0$ (LS: $\sigma=1.0$)}  \\
\end{tabular}
\caption{Performance comparison (validation accuracy) between different DP optimization algorithms in training CNN for MNIST classification with a fixed $\delta=10^{-5}$.}
\label{fig-CNN-MNIST-Acc-Eps}
\end{figure}


Next, we consider the effects of the LS constant ($\sigma$) and the learning rate in training the DP-CNN 
for MNIST classification. We fixed the NM to be $10$, and run $60$ epochs of DP-SGD and DP-LSSGD with different $\sigma$ and different learning rate. We show the comparison of DP-SGD with DP-LSSGD with different $\sigma$ in the left panel of Figure~\ref{fig-CNN-MNIST-Different-Parameters}, and we see that as $\sigma$ increases it becomes more stable in training CNNs with DP-guarantee even though initially it becomes slightly slower. In the middle panel of Figure~\ref{fig-CNN-MNIST-Different-Parameters}, we plot the evolution of validation accuracy curves of the DP-CNN trained by DP-SGD and DP-LSSGD with different learning rate, where the solid lines represent results for DP-LSSGD and dashed lines for DP-SGD. DP-LSSGD outperforms DP-SGD in all learning rates tested, and DP-LSSGD is much more stable than DP-SGD when a larger learning rate is used.

Finally, we go back to the accuracy degradation problem raised in Figure~\ref{Degrade-performance}. As shown in Figure~\ref{fig:Training-Testing-Loss-Softmax}, LS can efficiently reduce both training and validation losses in training multi-class logistic regression for MNIST classification. Moreover, as shown in the right panel of Figure~\ref{fig-CNN-MNIST-Different-Parameters}, DP-LSSGD can improve the testing accuracy of the CNN used above significantly. In particular, DP-LSSGD improves the testing accuracy of CNN by $3.2$\% and $5.0$\% for $(0.4, 10^{-5})$ and $(0.2, 10^{-5})$, respectively, on top of DP-SGD. DP-LSAdam can further boost test accuracy. {All the accuracies associated with any given privacy budget 
in Figure~\ref{fig-CNN-MNIST-Different-Parameters} (right panel), are the optimal ones searched over the results obtained in the above experiments with different learning rate, number of epochs, and NM.}


\subsection{CNN for CIFAR10 Classification}
In this section, we will show that LS can also improve the utility of the DP-CNN trained by DP-SGD and DP-Adam for CIFAR10 classification. We simply replace the CNN architecture used above for MNIST classification with the benchmark architecture in the Tensorflow tutorial \footnote{github.com/tensorflow/models/tree/master/tutorials/image/cifar10} for CIFAR10 classification. Also, we use the same set of parameters as that used for training DP-CNN for MNIST classification except we fixed the noise multiplier to be $2.0$ and clip the gradient $\ell_2$ norm to $3$. As shown in Figure~\ref{fig-CIFAR10}, 
LS can significantly improve the validation accuracy of the model trained by DP-SGD and DP-Adam, and the DP guarantee for all these algorithms are the same (dashed line in Figure~\ref{fig-CIFAR10}).

\begin{figure}[!ht]
\centering
\begin{tabular}{c}
\includegraphics[clip, trim=0.1cm 5cm 0.1cm 6cm,width=0.4\columnwidth]{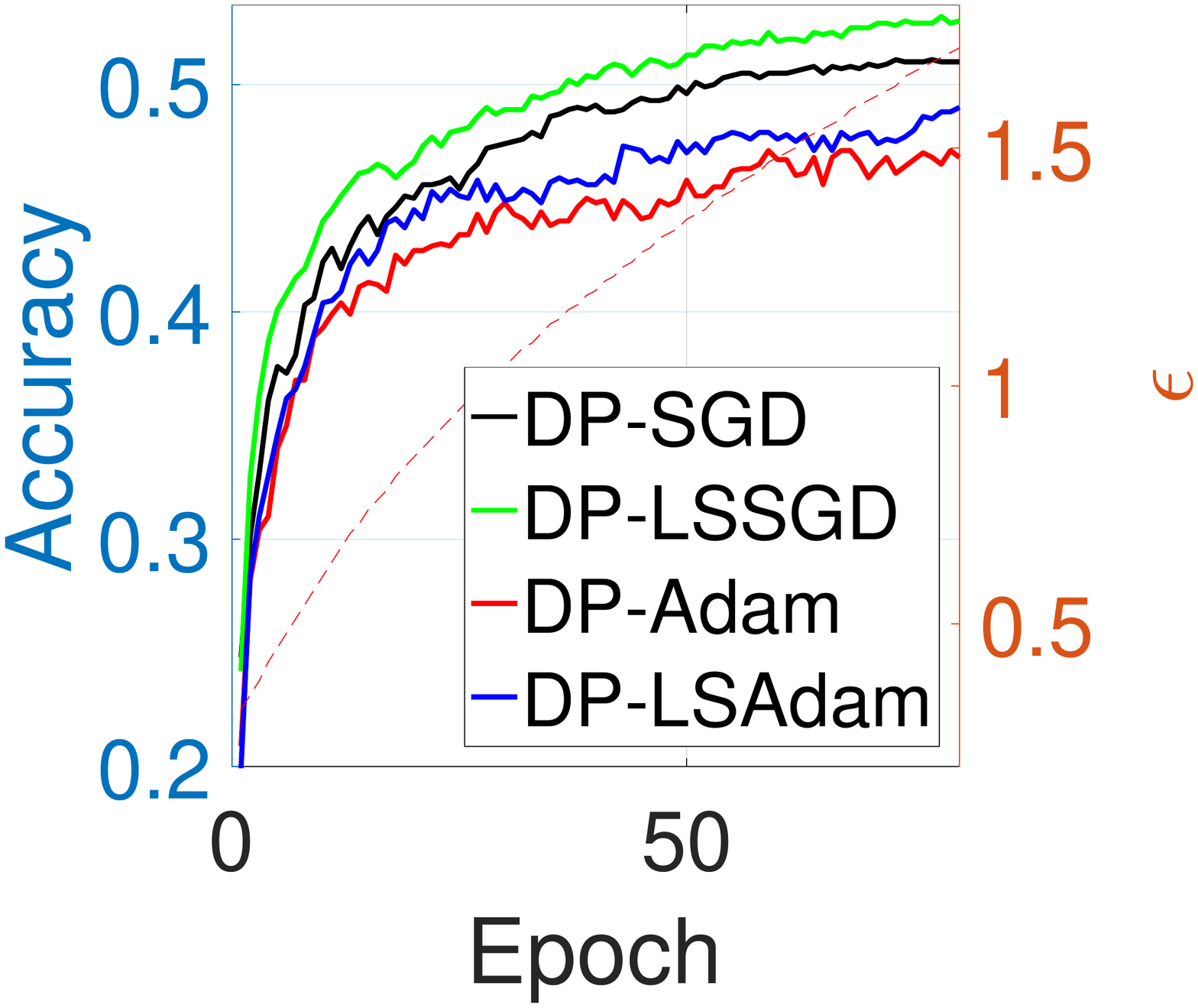}\\
\end{tabular}
\caption{Performance comparison between different differentially private optimization algorithms in training CNN for CIFAR10 classification with a fixed $\delta=10^{-5}$.}
\label{fig-CIFAR10}
\end{figure}

\begin{figure}[!ht]
\centering
\begin{tabular}{ccc}
\includegraphics[clip, trim=0.1cm 5cm 0.1cm 6.5cm,width=0.25\columnwidth]{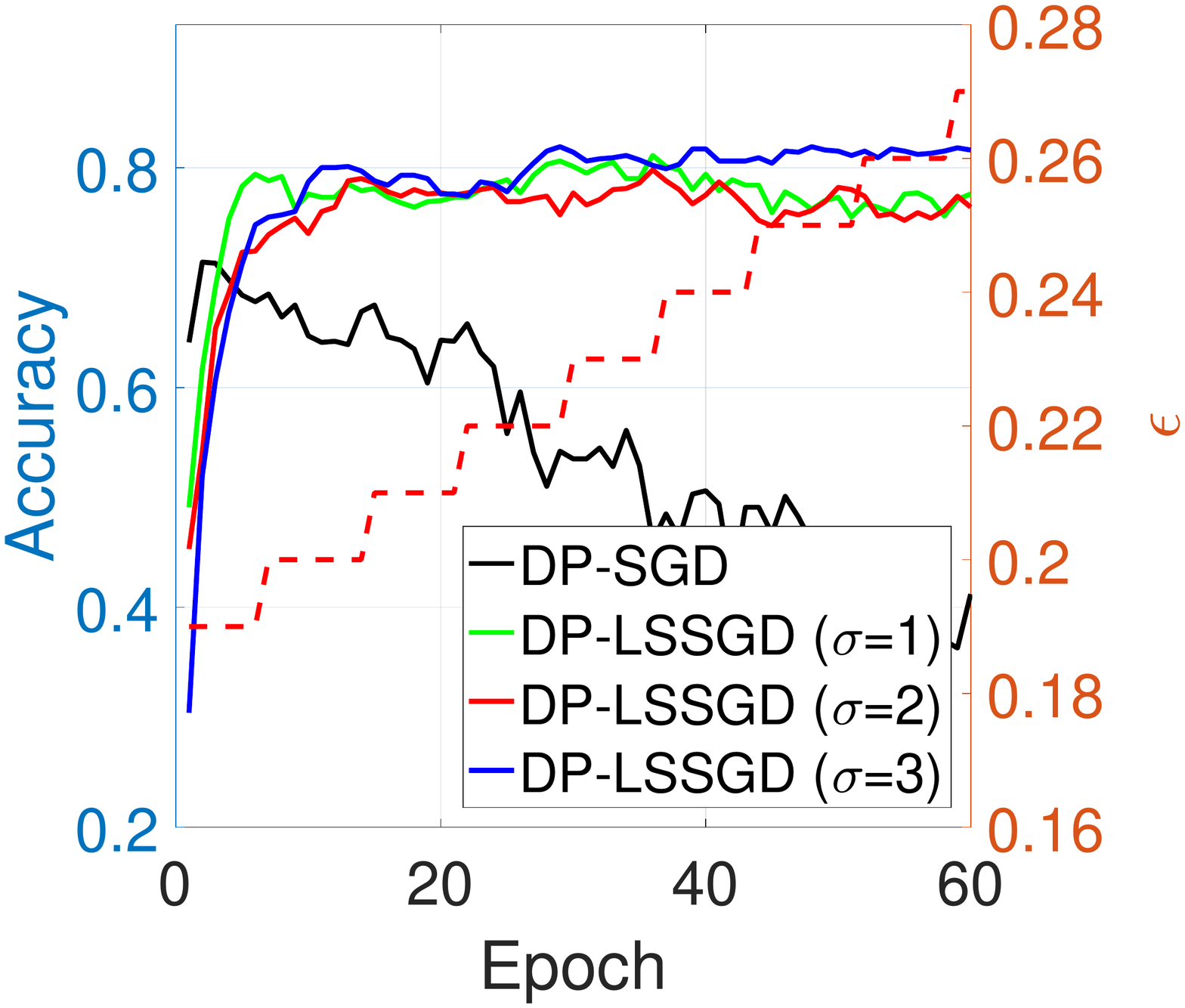}&
\includegraphics[clip, trim=0.1cm 5cm 0.1cm 6.5cm,width=0.25\columnwidth]{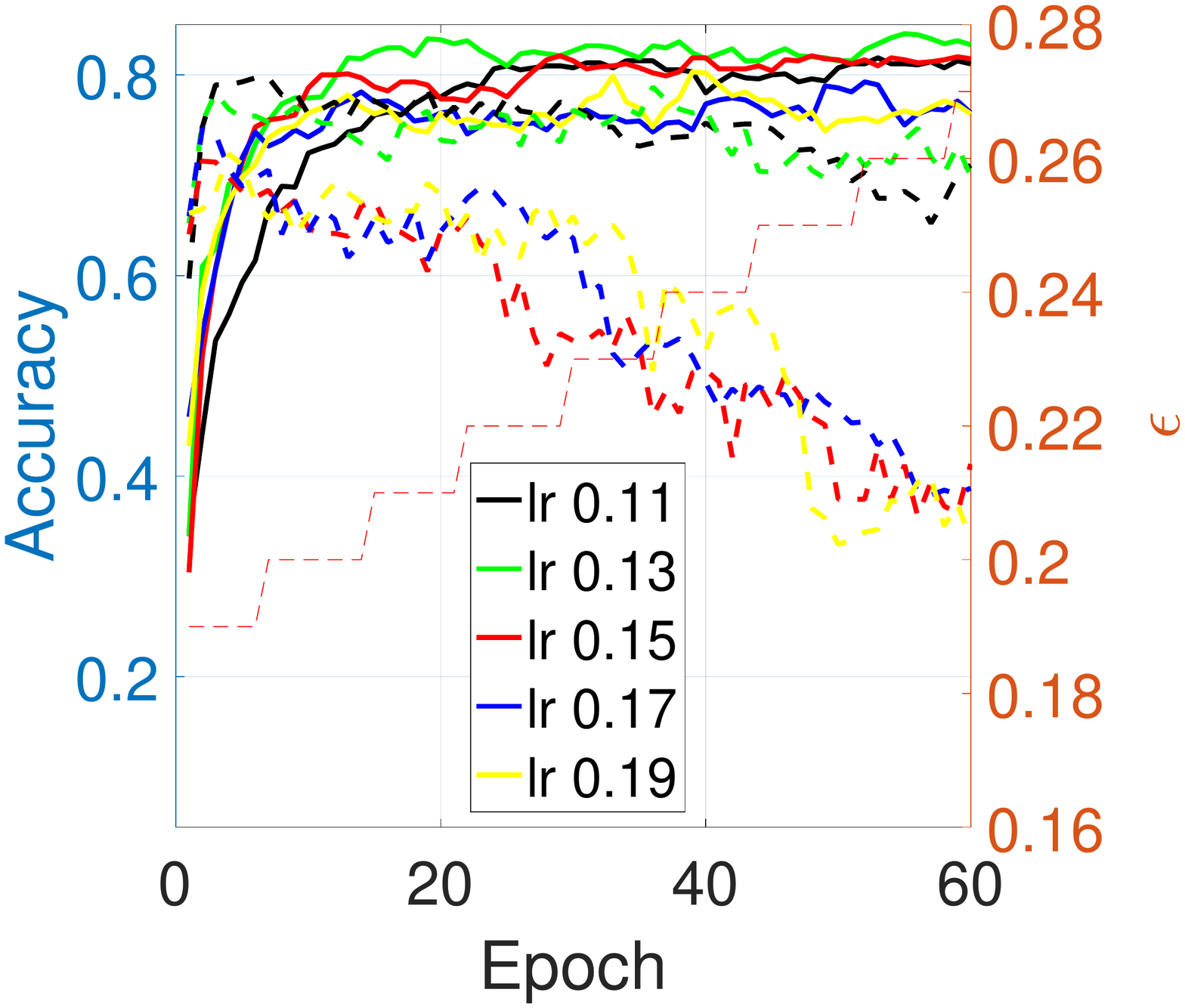}&
\includegraphics[clip, trim=0.1cm 5cm 0.1cm 4.cm,width=0.25\columnwidth]{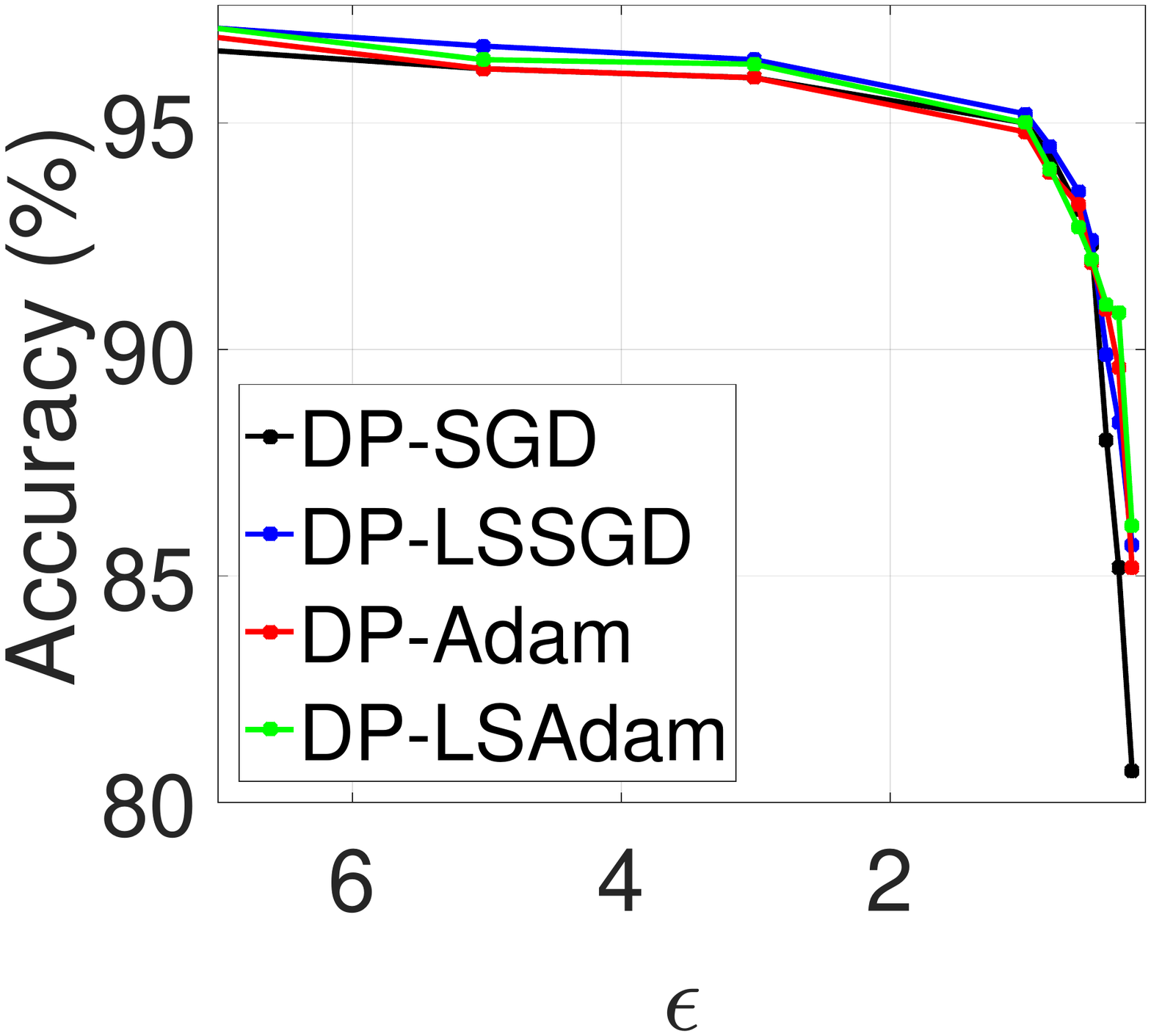}\\ 
Diff LS ($\sigma$)  
&  Diff Learning Rate & $\epsilon$ vs. Accuracy \\
\end{tabular}
\caption{Left \& middle panels: Contrasting performance (validation acc) of DP-SGD and DP-LSSGD with different $\sigma$ and different learning rate. Right panel: $\epsilon$ vs. Testing accuracy of the private models trained by different DP-optimization algorithms with a fixed $\delta=10^{-5}$.}
\label{fig-CNN-MNIST-Different-Parameters}
\end{figure}

\section{Conclusions}\label{Conclusion}
In this paper, we integrated Laplacian smoothing with DP-SGD for privacy-presrving ERM.
The resulting algorithm is simple to implement and the extra computational cost compared with the DP-SGD is almost negligible. We show that DP-LSSGD can improve the utility of the trained private ML models both numerically and theoretically. It is straightforward to combine LS with other variance reduction technique, e.g., SVRG \citep{SVRG}.

\section*{Acknowledgments}
This material is based on research sponsored by the National Science Foundation under grant number DMS-1924935 and DMS-1554564 (STROBE). The Air Force Research Laboratory under grant numbers FA9550-18-0167 and MURI FA9550-18-1-0502, the Office of Naval Research under grant number N00014-18-1-2527. QG is partially supported by the National Science Foundation under grant number SaTC-1717950.


\newpage
\appendix

\section{Proof of the Main Theorems}

\subsection{Privacy Guarantee}
To prove the privacy guarantee in Theorem \ref{Theorem-privacy-guarantee}, we first introduce the following $\ell_2$-sensitivity.
\begin{definition}[$\ell_2$-Sensitivity] \label{L2-Sensitivity}
For any given function $f(\cdot)$, the $\ell_2$-sensitivity of $f$ is defined by
$$
\Delta (f) = \max_{\|S-S'\|_1=1} \|f(S) - f(S')\|_2,
$$
where $\|S-S'\|_1=1$ means the data sets $S$ and $S'$ differ in only one entry.
\end{definition}

We will adapt the concepts and techniques of R\'enyi DP (RDP) to prove the DP-guarantee of the proposed DP-LSSGD.

\begin{definition}[RDP]\label{RDP}
For $\alpha>1 \ \mbox{and}\ \rho>0$, a randomized mechanism $\cM:\cS^n\rightarrow\cR$ satisfies $(\alpha, \rho)$-R\'enyi DP, i.e., $(\alpha, \rho)$-RDP, if for all adjacent datasets $S, S^\prime \in \cS^n$ differing by one element, we have 
\begin{align*}
D_{\alpha}\big(\cM(S)||\cM(S^\prime)\big):=\frac{1}{\alpha-1}\log \EE\bigg(\frac{\cM(S)}{\cM(S^\prime)}\bigg)^\alpha\leq \rho,
\end{align*}
where the expectation is taken over $\cM(S^\prime)$.
\end{definition}

\begin{lemma}\citep{wang2019efficient}\label{lemma:GaussianM_RDP}
	Given a function $q:\cS^n\rightarrow\cR$, the Gaussian Mechanism $\cM=q(S)+\ub$, where $\ub\sim N(0,\sigma^2\Ib)$, satisfies $(\alpha,\alpha\Delta^2(q)/(2\sigma^2))$-RDP. In addition, if we apply the mechanism $\cM$ to a subset of samples using uniform sampling without replacement, $\cM$ satisfies $(\alpha, 5\tau^2\Delta^2(q)\alpha/\sigma^2)$-RDP given $\sigma^{\prime2}=\sigma^2/\Delta^2(q)\geq 1.5$, $\alpha\leq \log(1/\tau\big(1+\sigma^{\prime2})\big)$, where $\tau$ is the subsample rate.
\end{lemma}

\begin{lemma}\citep{mironov2017renyi}\label{lemma:com_post}
	If $k$ randomized mechanisms $\cM_i:\cS^n\rightarrow\cR$, for $i\in[k]$, satisfy $(\alpha,\rho_i)$-RDP, then their composition $\big(\cM_1(S),\ldots,\cM_k(S)\big)$ satisfies $(\alpha,\sum_{i=1}^k\rho_i)$-RDP. Moreover, the input of the $i$-th mechanism can be based on outputs of the previous $(i-1)$ mechanisms.
\end{lemma}

\begin{lemma}\label{lemma:RDP_to_DP}
	If a randomized mechanism $\cM: \cS^n\rightarrow\cR$ satisfies $(\alpha,\rho)$-RDP, then $\cM$ satisfies $(\rho+\log(1/\delta)/(\alpha-1),\delta)$-DP for all $\delta\in(0,1)$.
\end{lemma}

With the definition (Def.~\ref{RDP}) and guarantees of RDP (Lemmas \ref{lemma:GaussianM_RDP} and \ref{lemma:com_post}), and the connection between RDP and $(\epsilon, \delta)$-DP (Lemma \ref{lemma:RDP_to_DP}), we can prove the following DP-guarantee for DP-LSSGD.

\begin{proof}[Proof of Theorem \ref{Theorem-privacy-guarantee}]
Let us denote the update of DP-SGD and DP-LSSGD at the $k$-th iteration starting from any given points $\wb^k$ and $\tilde \wb^k$, respectively, as
\begin{equation}
\label{DP-SGD-k}
\wb^{k+1} = \wb^k - \eta_k \bigg(\frac{1}{b}\sum_{i_k\in \mathcal{B}_k}\nabla  f_{i_k}(\wb^k) + \nb\bigg),
\end{equation}
and 
\begin{equation}
\label{DP-LSSGD-k}
\tilde \wb^{k+1} = \tilde \wb^k - \eta_k \Ab_\sigma^{-1}\bigg(\frac{1}{b}\sum_{i_k\in \mathcal{B}_k}\nabla f_{i_k}(\tilde \wb^k) + \nb\bigg),
\end{equation}
where $\mathcal{B}_k$ is a mini batch that are drawn uniformly from $[n]$, and $|\mathcal{B}_k| = b$ is the mini batch size.

We will show that with the aforementioned Gaussian noise $\mathcal{N}(0, \nu^2)$ for each coordinate of $\nb$, the output of DP-SGD, $\tilde \wb$, after $T$ iterations is $(\epsilon, \delta)$-DP. Let us consider the mechanism $\hat\cM_k=\frac{1}{b}\sum_{i_k\in \mathcal{B}_k}\nabla f_{i_k}(\wb^k)+\nb$, and $ \cM_k=\frac{n}{b}\nabla F(\wb^k)+\nb$ with the query $\mathbf{q}_k=\frac{n}{b}\nabla F(\wb^k)$. We have the $\ell_2$-sensitivity of $\mathbf{q}_k$ as $\Delta(\mathbf{q}_k)=\|\nabla f_{i_k}(\wb^k)-\nabla f_{i_k^\prime}(\wb^k)\|_2\leq \frac{2G}{b}$. According to Lemma \ref{lemma:GaussianM_RDP}, 
if we add noise with variance 
\begin{align*}
    \nu^2=\frac{20T\alpha G^2}{n^2\epsilon\mu},
\end{align*}
the mechanism $\cM_k$ will satisfy $\big(\alpha,(n^2\epsilon \mu /b^2)/\big(10T\big)\big)$-RDP. By post-processing theorem, we immediately have that under the same noise, $ \tilde{\cM}_k=\Ab_{\sigma}^{-1}(\nabla F(\wb^k)+\nb)$ also satisfies $\big(\alpha,(n^2\epsilon\mu/b^2)/\big(10T\big)\big)$-RDP. 
According to Lemma \ref{lemma:GaussianM_RDP},
$\hat{\cM}_k$ will satisfy $\big(\alpha,\mu\epsilon/T\big)$-RDP provided that $\nu^2/\Delta(\mathbf{q}_k)^2\geq 1.5$, because $\tau = b/n$. Let $\alpha=\log(1/\delta)/\big((1-\mu)\epsilon\big)+1$, we obtain that $\hat{\cM}_k$ satisfies $\big(\log(1/\delta)/\big((1-\mu)\epsilon\big)+1,\mu\epsilon/T\big)$-RDP as long as we have
\begin{align*}
	\frac{\nu^2}{\Delta(\mathbf{q}_k)^2}=\frac{5T\alpha b^2}{n^2\epsilon\mu}\geq 1.5.
	\end{align*}
In addition, we have
\begin{align*}
    \frac{1}{\tau\big(1+\nu^{2}/\Delta(\mathbf{q}_k)^2\big)}=\frac{\mu n^3\epsilon}{5b^3T\alpha+\mu b n^2\epsilon},
\end{align*}
which implies that $\alpha=\log(1/\delta)/\big((1-\mu)\epsilon\big)+1 \leq \log\big(\mu n^3\epsilon/(5b^3T\alpha+\mu b n^2\epsilon)\big)$.
Therefore, according to Lemma \ref{lemma:com_post}, we have $\wb^k$ satisfies $\big(\log(1/\delta)/\big((1-\mu)\epsilon\big)+1,k\mu\epsilon/T\big)$-RDP. Finally, by Lemma \ref{lemma:RDP_to_DP}, we have $\wb^k$ satisfies $\big(k\mu\epsilon/T+(1-\mu)\epsilon,\delta\big)$-DP. Therefore, the output of DP-SGD, $\tilde \wb$, is $(\epsilon, \delta)$-DP.
\end{proof}

\begin{remark}
In the above proof, we used the following estimate of the $\ell_2$ sensitivity
$$
\Delta(\mathbf{q}_k)=\|\Ab_{\sigma}^{-1}\nabla f_i(\wb^k)-\Ab_{\sigma}^{-1}\nabla f_{i^\prime}(\wb^k)\|_2/n\leq 2G/n.
$$
Indeed, let $\mathbf{g}=\nabla f_i(\wb^k)-\nabla f_{i^\prime}(\wb^k)$ and $\mathbf{d} = \Ab_{\sigma}^{-1} \mathbf{g}$, then according to \cite{LSGD:2018} we have
$$
\|\mathbf{d}\|_2 + 2\sigma \frac{\|\mathbf{D}_+\mathbf{d}\|_2^2}{d} + \sigma^2 \frac{\|\mathbf{L}\mathbf{d}\|_2^2}{d}= \|\mathbf{g}\|_2,
$$
where $d$ is the dimension of $\mathbf{d}$, and 
\begin{equation*}
\mathbf{D}_+ = \begin{bmatrix}
-1   & 1 &  0&\dots &0& 0 \\
0     & -1 & 1 & \dots &0&0 \\
0 & 0  & -1 & \dots & 0 & 0 \\
\dots     & \dots & \dots &\dots & \dots & \dots\\
1     &0& 0 & \dots &0 & -1
\end{bmatrix}.
\end{equation*}
Moreover, if we assume the $\mathbf{g}$ is randomly sampled from a unit ball in a high dimensional space, then a high probability estimation of the compression ratio of the $\ell_2$ norm can be derived from Lemma.~\ref{expectedl2Gu2}.

Numerical experiments show that $\|\Ab_{\sigma}^{-1}\nabla f_i(\wb^k)-\Ab_{\sigma}^{-1}\nabla f_{i^\prime}(\wb^k)\|_2$ is much less than $\|\nabla f_i(\wb^k)-\nabla f_{i^\prime}(\wb^k)\|_2$, so for the above noise, it can give much stronger privacy guarantee.
\end{remark}

\subsection{Utility Guarantee -- Convex Optimization}
To prove the utility guarantee for convex optimization, we first show that 
the LS operator compresses the $\ell_2$ norm of any given Gaussian random vector with a specific ratio in expectation. 

\begin{lemma}\label{expectedl2Gu1}
Let $\bx \in \RR^d$ be the standard Gaussian random vector. Then
\begin{align*}
\mathbb{E}\|\bx\|_{\Ab_{\sigma}^{-1}}^2= \sum_{i=1}^d \frac{1}{1+2\sigma - 2\sigma \cos(2\pi i/d) },
\end{align*}
where $\|\bx\|_{\Ab_{\sigma}^{-1}}^2 \doteq \la\bx, \Ab_\sigma^{-1}\bx\ra $ is the square of the induced norm of $\bx$ by the matrix $\Ab_\sigma^{-1}$.
\end{lemma}
\begin{proof}[Proof of Lemma~\ref{expectedl2Gu1}]
Let the eigenvalue decomposition of $\Ab_{\sigma}^{-1}$ be $\Ab_{\sigma}^{-1} = \Ub\bLambda \Ub^T$, where $\bLambda$ is a diagonal matrix with $\Lambda_{ii} = \frac{1}{1+2\sigma - 2\sigma \cos(2\pi i/d) }$
We have
\begin{align*}
    \mathbb{E}\|\bx\|_{\Ab_{\sigma}^{-1}}^2 &= \EE[\Tr(\bx^\top \Ub \bLambda \Ub^\top \bx)]\\
    &= \sum_{i=1}^d \Lambda_{ii}\\
    &= \sum_{i=1}^d \frac{1}{1+2\sigma - 2\sigma \cos(2\pi i/d) } = \gamma. 
\end{align*}
\end{proof}

\begin{proof}[Proof of Theorem \ref{Convex-Utility}]
Recall that we have the following update rule $\wb^{k+1}=\wb^k-\eta_k\Ab_\sigma^{-1}(\nabla f_{i_k}(\wb^k)+\nb)$, where $i_k$ are drawn uniformly from $[n]$, and $\nb\sim \mathcal{N}(0, \nu^2\Ib)$. Let $\nabla f_{\mathcal{B}_k}=\sum_{i_k\in \mathcal{B}_k}\nabla f_{i_k}(\wb^k)/b$, observe that
\begin{align*}
    \|\wb^{k+1}-\wb^*\|_{\Ab_\sigma}^2&=\|\wb^{k}-\eta_k\Ab_\sigma^{-1}(\nabla f_{\mathcal{B}_k}(\wb^k)+\nb)-\wb^*\|_{\Ab_\sigma}^2\\
    &=\|\wb^{k}-\wb^*\|_{\Ab_\sigma}^2+\eta_k^2\big(\big\|\Ab_\sigma^{-1}\big(\nabla f_{\mathcal{B}_k}(\wb^k)-\nabla F(\wb^k)+\nabla F(\wb^k)\big)\big\|_{\Ab_\sigma}^2+\|\Ab_\sigma^{-1}\nb\|_{\Ab_\sigma}^2\\&\qquad+2\la\Ab_\sigma^{-1}\nabla f_{\mathcal{B}_k}(\wb^k),\nb\ra\big)
    -2\eta_k\la\nabla f_{\mathcal{B}_k}(\wb^k)+\nb,\wb^k-\wb^*\ra.
\end{align*}
Taking expectation with respect to $\mathcal{B}_k$ and $\nb$ given $\wb^k$, we have
\begin{align*}
    \EE \|\wb^{k+1}-\wb^*\|_{\Ab_\sigma}^2&=\EE\|\wb^{k}-\wb^*\|_{\Ab_\sigma}^2-2\eta_k\EE\la\nabla F(\wb^{k}),\wb^k-\wb^*\ra+\eta_k^2\EE\|\nabla f_{\mathcal{B}_k}(\wb^k)-\nabla F(\wb^{k})\|_{\Ab_\sigma^{-1}}^2\\
    &\qquad+\eta_k^2\EE\|\nabla F(\wb^{k})\|_{\Ab_\sigma^{-1}}^2+\eta_k^2\EE\|\nb\|_{\Ab_\sigma^{-1}}^2.
\end{align*}
In addition, we have
\begin{align}\label{eq:concen_1}
    \EE\|\nabla f_{\mathcal{B}_k}(\wb^k)-\nabla F(\wb^{k})\|_{\Ab_\sigma^{-1}}^2\leq\EE\|\nabla f_{\mathcal{B}_k}(\wb^k)-\nabla F(\wb^{k})\|_2^2 \leq \frac{G^2}{b},
\end{align}
and 
\begin{align}\label{eq:concen_2}
    \bigg(1-\frac{L\eta_k}{2}\bigg)\eta_k\|\nabla F(\wb^{k})\|_2^2\leq F(\wb^{k})-F(\wb^{*}),
\end{align}
which implies that 
\begin{align*}
    \eta_k^2\EE\|\nabla F(\wb^{k})\|_{\Ab_\sigma^{-1}}^2\leq \eta_k^2\EE\|\nabla F(\wb^{k})\|_2^2\leq \bigg(\frac{2}{2-L\eta_k}\bigg)\eta_k\EE\big(F(\wb^{k})-F(\wb^{*})\big)\leq \frac{4}{3}\eta_k \EE\big(F(\wb^{k})-F(\wb^{*})\big),
\end{align*}
where the last inequality is due to the fact that $\eta_t\leq 1/(2L)$. Therefore, we have
\begin{align*}
    \EE \|\wb^{k+1}-\wb^*\|_{\Ab_\sigma}^2&\leq \EE\|\wb^{k}-\wb^*\|_{\Ab_\sigma}^2-\frac{2}{3}\eta_k\EE\big(F(\wb^k)-F(\wb^*)\big)+\eta_k^2\big(G^2/b+\gamma  d\nu^2\big),
\end{align*}
where the inequality is due to the convexity of $F$, and Lemma \ref{expectedl2Gu1}. It implies that 
\begin{align*}
   \frac{2}{3}\eta_k\EE\big(F(\wb^k)-F(\wb^*)\big)\leq\big(\EE\|\wb^{k}-\wb^*\|_{\Ab_\sigma}^2-\EE \|\wb^{k+1}-\wb^*\|_{\Ab_\sigma}^2\big)+\eta_k^2(G^2/b+\gamma d\nu^2).
\end{align*}
Now taking the full expectation and summing up over $T$ iterations, we have
\begin{align*}
    \sum_{k=0}^{T-1}\frac{2}{3}\eta_k\EE\big(F(\wb^k)-F(\wb^*)\big)&\leq D_\sigma+\sum_{k=0}^{T-1}\eta_k^2(G^2/b+\gamma d\nu^2),
\end{align*}
where $D_\sigma=\|\wb^0-\wb^*\|_{\Ab_\sigma}^2$. Let $v_k=\eta_k/\big(\sum_{k=0}^{T-1}\eta_k\big)$, we have
\begin{align*}
    \sum_{k=0}^{T-1}v_k\EE\big(F(\wb^k)-F(\wb^*)\big)&\leq \frac{D_\sigma+\sum_{k=0}^{T-1}\eta_k^2(G^2/b+\gamma d\nu^2)}{2\sum_{k=0}^{T-1}\eta_k/3}.
\end{align*}
According to the definition of $\tilde \wb$ and the convexity of $F$, we obtain
\begin{align*}
    \EE\big(F(\tilde \wb)-F(\wb^*)\big)&\leq  \frac{D_\sigma+\sum_{k=0}^{T-1}\eta_k^2(G^2/b+\gamma d\nu^2)}{2\sum_{k=0}^{T-1}\eta_k/3}\\
    &\leq\frac{D_\sigma+\sum_{k=0}^{T-1}\eta_k^2G^2/b}{2\sum_{k=0}^{T-1}\eta_k/3}+\frac{\sum_{k=0}^{T-1}\eta_k^2}{2\sum_{k=0}^{T-1}\eta_k/3}\cdot\frac{20\gamma dTG^2\log(1/\delta)}{n^2\epsilon^2\mu(1-\mu)}.
\end{align*}

Let $\eta=1/\sqrt{T}$ and $T=C_1(D_\sigma+G^2/b)n^2\epsilon^2/\big(\gamma d G^2 \log(1/\delta)\big)$, we can obtain that 
\begin{align*}
    \EE\big(F(\tilde \wb)-F(\wb^*)\big)\leq \frac{C_2G\sqrt{\gamma (D_\sigma+G^2/b)d\log(1/\delta)}}{n\epsilon},
\end{align*}
where $C_1,C_2$ are universal constants.
\end{proof}

\subsection{Utility Guarantee -- Nonconvex Optimization}
To prove the utility guarantee for nonconvex optimization, we need the following lemma, which shows that the LS operator compresses the $\ell_2$ norms of any given Gaussian random vector with a specific ratio in expectation. 

\begin{lemma}\label{expectedl2Gu2}
Let $\bx \in \RR^d$ be the standard Gaussian random vector. Then
\begin{align*}
\mathbb{E}\|\Ab_{\sigma}^{-1}\bx\|_2^2= \sum_{i=1}^d \frac{1}{(1+2\sigma - 2\sigma \cos(2\pi i/d))^2 } .
\end{align*}
\end{lemma}

\begin{proof}[Proof of Lemma~\ref{expectedl2Gu2}]
Let the eigenvalue decomposition of $\Ab_{\sigma}^{-1}$ be $\Ab_{\sigma}^{-1} = \Ub\bLambda \Ub^T$, where $\bLambda$ is a diagonal matrix with $\Lambda_{ii} = \frac{1}{1+2\sigma - 2\sigma \cos(2\pi i/n) }$
We have
\begin{align*}
    \mathbb{E}\|\Ab_{\sigma}^{-1}\bx\|_2^2 &= \EE[\Tr(\bx^\top \Ub \bLambda \Ub^\top \Ub \bLambda \Ub^\top \bx)] \\
    &= \EE[\Tr(\bx^\top \Ub \bLambda^2 \Ub^\top \bx)] \\
    &=  \sum_{i=1}^d \Lambda_{ii}^2\\
    &= \sum_{i=1}^d \frac{1}{(1+2\sigma - 2\sigma \cos(2\pi i/d))^2 } = \beta.
\end{align*}
\end{proof}

\begin{proof}[Proof of Theorem \ref{Nonconvex-Utility}]
Recall that we have the following update rule $\wb^{t+1}=\wb^k-\eta_k\Ab_{\sigma}^{-1}(\nabla f_{i_k}(\wb^k)+\mathbf{n})$, where $i_k$ are drawn uniformly from $[n]$, and $\nb\sim \mathcal{N}(0, \nu^2\Ib)$.
Let $\nabla f_{\mathcal{B}_k}=\sum_{i_k\in \mathcal{B}_k}\nabla f_{i_k}(\wb^k)/b$, since $F$ is $L$-smooth, we have
\begin{align*}
    F(\wb^{k+1}) &\leq F(\wb^{k}) + \la \nabla F(\wb^{k}),\wb^{k+1}-\wb^k \ra + \frac{L}{2} \|\wb^{k+1}-\wb^k\|_2^2\\
    &=F(\wb^{k}) -\eta_k \la \nabla F(\wb^{k}),\Ab_{\sigma}^{-1}(\nabla f_{\mathcal{B}_k}(\wb^k)+\mathbf{n}) \ra \\
    &\qquad + \frac{\eta_k^2L}{2} \Big(\big\|\Ab_{\sigma}^{-1}\big(\nabla f_{\mathcal{B}_k}(\wb^k)-\nabla F(\wb^{k})+\nabla F(\wb^{k})\big)\big\|_2^2+\|\Ab_{\sigma}^{-1}\mathbf{n}\|_2^2+2\la\Ab_{\sigma}^{-1}\nabla f_{\mathcal{B}_k}(\wb^k),\Ab_{\sigma}^{-1}\mathbf{n}\ra\Big).
\end{align*}
Taking expectation with respect to $\mathcal{B}_k$ and $\mathbf{n}$ given $\wb^k$, we have
\begin{align*}
    \EE F(\wb^{k+1}) &\leq \EE F(\wb^{k}) -\eta_k \EE \la \nabla F(\wb^{k}),\Ab_{\sigma}^{-1}\nabla f_{\mathcal{B}_k}(\wb^k) \ra + \frac{\eta_k^2L}{2} \Big(\EE\|\Ab_{\sigma}^{-1}\big(\nabla f_{\mathcal{B}_k}(\wb^k)-\nabla F(\wb^{k})\big)\|_2^2\\
    &\qquad+\EE\|\Ab_{\sigma}^{-1}\nabla F(\wb^{k})\big\|_2^2+\EE\|\Ab_{\sigma}^{-1}\mathbf{n}\|_2^2\Big)\\
    &\leq \EE F(\wb^{k})-\eta_k\Big(1-\frac{\eta_k L}{2}\Big)\EE\|\nabla F(\wb^k)\|_{\Ab_{\sigma}^{-1}}^2+\frac{\eta_k^2L}{2}(G^2/b+d\beta \nu^2)\\
    &\leq \EE F(\wb^{k})-\frac{\eta_k}{2}\EE\|\nabla F(\wb^k)\|_{\Ab_{\sigma}^{-1}}^2+\frac{\eta_k^2L(G^2+d\beta\nu^2)}{2},
\end{align*}
where the second inequality uses Lemma~\ref{expectedl2Gu2}, the inequality \eqref{eq:concen_1}, and the last inequality is due to $1-\eta_kL/2>1/2$. Now taking the full expectation and summing up over $T$ iterations, we have
\begin{align*}
    \EE F(\wb^T)\leq F(\wb^0)-\sum_{k=1}^{T-1}\frac{\eta_k}{2}\EE\|\nabla F(\wb^k)\|_{\Ab_{\sigma}^{-1}}^2+\sum_{k=1}^{T-1}\frac{\eta_k^2L(G^2/b+d\beta\nu^2)}{2}.
\end{align*}
If we choose fix step size, i.e., $\eta_k=\eta$, and rearranging the above inequality, and using $F(\wb^0)-\EE F(\wb^T) \leq F(\wb^0) - F(\wb^*)$, we get
\begin{align*}
    \frac{1}{T}\sum_{k=1}^{T-1}\EE\|\nabla F(\wb^k)\|_{\Ab_{\sigma}^{-1}}^2\leq \frac{2}{\eta T}\big(F(\wb^0)- F(\wb^*)\big)+\eta L(G^2/b+d\beta\nu^2),
\end{align*}
which implies that 
\begin{align*}
    \EE\|\nabla F(\tilde \wb)\|_{\Ab_{\sigma}^{-1}}^2&\leq \frac{2D_F}{\eta T}+\eta L(G^2/b+d\beta\nu^2)\\
    &\leq\frac{2D_F}{\eta T}+\eta L\bigg(G^2/b+\frac{20d\beta TG^2\log(1/\delta)}{n^2\epsilon^2\mu(1-\mu)}\bigg).
\end{align*}
Let $\eta=1/\sqrt{T}$ and $T=C_1(2D_F+LG^2/b)n^2\epsilon^2/\big(dL\beta G^2\log(1/\delta)\big)$, where $D_F=F(\wb^0)-F(\wb^*)$, we obtain
\begin{align*}
    \EE\|\nabla F(\tilde \wb)\|_{\Ab_{\sigma}^{-1}}^2\leq C_2\frac{ G\sqrt{\beta dL(2D_F+LG^2/b)\log(1/\delta)}}{n\epsilon},
\end{align*}
where $C_1,C_2$ are universal constants.
\end{proof}

\section{Calculations of $\beta$ and $\gamma$}
\subsection{Calculation of $\gamma$}
To prove Proposition~\ref{Gamma-Value}, we need the following two lemmas.
\begin{lemma}[Residue Theorem]
\label{Residual-Theorem}
Let $f(z)$ be a complex function defined on $\mathbb{C}$, then the residue of $f$ around the pole $z=c$ can be computed by the formula
\begin{equation}
\label{Residue}
{\rm Res}(f, c) = \frac{1}{(n-1)!}\lim_{z\rightarrow c} \frac{d^{n-1}}{dz^{z-1}}\left((z-c)^n f(z) \right).
\end{equation}
where the order of the pole $c$ is $n$. Moreover,
\begin{equation}
\label{Integral}
\oint f(z) dz = 2\pi i \sum_{c_i} Res(f, c_i),
\end{equation}
where $\{c_i\}$ be the set of pole(s) of $f(z)$ inside $\{z||z| < 1\}$.
\end{lemma}
The proof of Lemma~\ref{Residual-Theorem} can be found in any complex analysis textbook.

\begin{lemma}\label{lemma:DTFT}
For $0\leq \theta \leq 2\pi$, suppose 
$$
F(\theta) =  \frac{1}{1+2\sigma(1-\cos(\theta))},
$$
has the discrete-time Fourier transform of series $f[k]$. Then, for integer $k$,
$$
f[k] = \frac{\alpha^{|k|}}{\sqrt{4\sigma+1}}
$$
where 
$$
\alpha = \frac{2\sigma+1 - \sqrt{4\sigma+1}}{2\sigma}
$$
\end{lemma}
\begin{proof}
By definition, 
\begin{equation}\label{eqn:f_k_real}
f[k] = \frac{1}{2\pi} \int_{0}^{2\pi}  F(\theta) e^{ik\theta} \,d\theta=  \frac{1}{2\pi} \int_{0}^{2\pi}   \frac{e^{ik\theta}}{1+2\sigma(1-\cos(\theta))} \,d\theta.
\end{equation}
We compute \eqref{eqn:f_k_real} by using Residue theorem.
First, note that because $F(\theta)$ is real valued, $f[k]=f[-k]$; therefore, it suffices to compute \eqref{eqn:f_k_real}) for nonnegative $k$. Set $z=e^{i\theta}$. Observe that $\cos(\theta)=0.5(z+1/z)$ and $dz=iz d\theta$. Substituting in \eqref{eqn:f_k_real} and simplifying yields that
\begin{equation}\label{eqn:f_k_complex}
f[k] = \frac{-1}{2\pi i \sigma}\oint \frac{z^k}{(z-\alpha_{-})(z-\alpha_{+}) } \,dz,
\end{equation}
where the integral is taken around the unit circle, and $\alpha_{\pm}= \frac{2\sigma+1 \pm \sqrt{4\sigma+1}}{2\sigma}$ are the roots of quadratic $-\sigma z^2 +(2\sigma+1)z -\sigma$. Note that $\alpha_{-}$ lies within the unit circle; whereas, $\alpha_{+}$ lies outside of the unit circle. Therefore, because $k$ is nonnegative, $\alpha_{-}$ is the only singularity of the integrand in \eqref{eqn:f_k_complex} within the unit circle. A straightforward application of the Residue Theorem, i.e., Lemma~\ref{Residual-Theorem}, yields that 
$$
f[k] = \frac{- \alpha_{-}^{k}}{\sigma (\alpha_{-}-\alpha_{+})} = \frac{\alpha^{k}}{\sqrt{4\sigma+1}}.
$$  
This completes the proof.
\end{proof}

\begin{proof}[Proof of Proposition~\ref{Gamma-Value}]
First observe that we can re-write $\gamma$ as
\begin{equation}\label{Gamma}
\frac{1}{d}\sum_{j=0}^{d-1} \frac{1}{1+2\sigma(1-\cos(\frac{2\pi j}{d}))}.
\end{equation}
It remains to show that the above summation is equal to $\frac{1+\alpha^d}{(1-\alpha^d)\sqrt{4\sigma+1}}$.
This follows by lemmas \ref{lemma:DTFT} and standard sampling results in Fourier analysis (i.e. sampling $\theta$ at points $\{2\pi j/d\}_{j=0}^{d-1}$). Nevertheless, we provide the details here for completeness: Observe that that the inverse discrete-time Fourier transform of
$$
G(\theta) = \sum_{j=0}^{d-1}\delta\bigg(\theta-\frac{2\pi j }{d}\bigg).
$$
is given by
$$
g[k] = 
\begin{cases}
d/2\pi \qquad &\text{if $k$ divides $d$,}\\
0 \qquad  &\text{otherwise.}
\end{cases}
$$   
Furthermore, let 
$$
F(\theta) =  \frac{1}{1+2\sigma(1-\cos(\theta))},
$$
and use $f[k]$ to denote its inverse discrete-time Fourier transform.
Now,
\begin{align*}
\frac{1}{d}\sum_{j=0}^{d-1} \frac{1}{1+2\sigma(1-\cos(\frac{2\pi j}{d}))} &= \frac{1}{d} \int_0^{2\pi} F(\theta)G(\theta) \\
&= \frac{2\pi}{d}  \DTFT^{-1}[F\cdot G][0]  \\
&= \frac{2\pi}{d}  (\DTFT^{-1}[F] * \DTFT^{-1}[G])[0] \\
&=  \frac{2\pi}{d}   \sum_{r=-\infty}^{\infty} f[-r]g[r] \\
&=  \frac{2\pi}{d}  \sum_{\ell=-\infty}^{\infty} f[-\ell d]  \frac{d}{2\pi} \\
&=  \sum_{\ell=-\infty}^{\infty} f[-\ell d].
\end{align*}
The proof is completed by substituting the result of lemma \ref{lemma:DTFT} in the above sum and simplifying.
\end{proof}

We list some typical values of $\gamma$ in Table~\ref{Gamma-Value}.
\begin{table}[!ht]
\centering
\fontsize{10}{10}\selectfont
\begin{threeparttable}
\caption{The values of $\gamma$ corresponding to some $\sigma$ and $d$.}\label{Gamma-Table}
\begin{tabular}{cccccc}
\toprule[1.0pt]
$\sigma$   &\ \ \ \ \ \  1 \ \ \ \ \ \ &\ \ \ \ \ \ \ \  2\ \ \ \ \ \ \ \  &\ \ \ \ \ \ \ \  3\ \ \ \ \ \ \ \  &\ \ \ \ \ \ \ \  4\ \ \ \ \ \ \ \  &\ \ \ \ \ \ \ \  5\ \ \ \ \ \ \ \   \cr
\midrule[0.8pt]
$d=1000$    & 0.447 & 0.333 & 0.277 & 0.243 & 0.218 \cr
$d=10000$   & 0.447 & 0.333 & 0.277 & 0.243 & 0.218 \cr
$d=100000$  & 0.447 & 0.333 & 0.277 & 0.243 & 0.218 \cr
\bottomrule[1.0pt]
\end{tabular}
\end{threeparttable}
\end{table}

\subsection{Calculation of $\beta$}
The proof of Proposition~\ref{Beta-Value} is similar as the proof of Proposition~\ref{Gamma-Value}. The only difference is that we need to compute
\begin{equation}
\label{f_k_2}
f[k] = \frac{1}{2\pi}\int_0^{2\pi} \frac{e^{ik\theta}}{\left(1+2\sigma(1-\cos\theta)\right)^2}d\theta.
\end{equation}

By Residue theorem, for $k>0$ (note that $f[-k] = f[k]$
), we have
\begin{align*}
f[k] &= \frac{1}{2\pi}\int_0^{2\pi} \frac{e^{ik\theta}}{\left(1+2\sigma(1-\cos\theta)\right)^2}d\theta\\
&= \frac{1}{2\pi i}\oint \frac{z^{k+1}}{ (z + \sigma(2z-z^2-1))^2 } dz\\
&= \lim_{z\rightarrow \alpha^-} \frac{d}{dz} \left((z-\alpha^-)^2 \frac{z^{k+1}}{ (z + \sigma(2z-z^2-1))^2 } \right)\\
&= \lim_{z\rightarrow \alpha^-} \frac{d}{dz} \left(\frac{z^{k+1}}{ \sigma^2(z-\alpha^+)^2 } \right)\\
&= \frac{(k+1)\alpha^k}{4\sigma+1} + \frac{2\sigma\alpha^{k+1}}{(4\sigma+1)^{3/2}},\\
\end{align*}
where $\alpha_- = \frac{2\sigma+1-\sqrt{4\sigma+1}}{2\sigma}$.
Therefore, we have 
$$
\beta = \frac{2\alpha^{2d+1}-\xi\alpha^{2d}+2\xi d\alpha^d-2\alpha+\xi}{\sigma^2\xi^3(1-\alpha^d)^2}.
$$
We list some typical values of $\beta$ in Table~\ref{Beta-Value}.
\begin{table}[!ht]
\centering
\fontsize{10}{10}\selectfont
\begin{threeparttable}
\caption{The values of $\beta$ corresponding to some $\sigma$ and $d$.}\label{Beta-Table}
\begin{tabular}{cccccc}
\toprule[1.0pt]
$\sigma$   &\ \ \ \ \ \ \ \  1 \ \ \ \ \ \ \ \ &\ \ \ \ \ \ \ \  2\ \ \ \ \ \ \ \  &\ \ \ \ \ \ \ \  3\ \ \ \ \ \ \ \  &\ \ \ \ \ \ \ \  4\ \ \ \ \ \ \ \  &\ \ \ \ \ \ \ \  5\ \ \ \ \ \ \ \   \cr
\midrule[0.8pt]
$d=1000$    & 0.268 & 0.185 & 0.149 & 0.128 & 0.114 \cr
$d=10000$   & 0.268 & 0.185 & 0.149 & 0.128 & 0.114 \cr
$d=100000$  & 0.268 & 0.185 & 0.149 & 0.128 & 0.114 \cr
\bottomrule[1.0pt]
\end{tabular}
\end{threeparttable}
\end{table}

\section{Laplacian Smoothing and Diffusion Equation}
Let $u(x, t)$ be a function defined on the space-time domain $[0, 1]\times [0, +\infty)$, suppose it satisfies the following diffusion equation with the Neumann boundary condition
\begin{eqnarray}
\label{Eq:Diffusion}
\begin{cases}
\frac{\partial u}{\partial t} = \frac{\partial^2 u}{\partial x^2},\ \ (x, t)\in [0, 1]\times [0, +\infty),\\
\frac{\partial u(0, t)}{\partial x} = \frac{\partial u(1, t)}{\partial x} = 0,\ t\in [0, +\infty)\\
u(x, 0) = f(x),\ \ x\in [0, 1]
\end{cases}
\end{eqnarray}

If we apply the backward Euler in time and central finite difference in space to discretize the governing equation in \eqref{Eq:Diffusion}, we get
$$
\mathbf{v}^{\Delta t} - \mathbf{v}^{0} = \Delta t\mathbf{L}\mathbf{v}^{\Delta t},
$$
where $\mathbf{v}^0$ is the discretization of $f(x)$, and $\mathbf{v}^{\Delta t}$ is the numerical solution of \eqref{Eq:Diffusion} at time $\Delta t$. Therefore, we have
$$
\mathbf{v}^{\Delta t} = (I-\Delta t\mathbf{L})^{-1}\mathbf{v}^0,
$$
which is the LS with $\sigma=\Delta t$.

\end{document}